\documentclass[10pt,conference]{IEEEtran}
\IEEEoverridecommandlockouts
\usepackage{cite}
\usepackage{algorithm}
\usepackage{amsmath,amssymb,amsfonts}
\usepackage{graphicx}
\usepackage{placeins}
\usepackage{amsthm}
\usepackage{booktabs}
\usepackage{multirow}
\usepackage{algpseudocode}
\usepackage{graphicx}
\usepackage{subfigure}
\usepackage{thmtools}
\usepackage[T1]{fontenc}
\usepackage{helvet}

\newtheorem{theorem}{Theorem}[section]

\usepackage{textcomp}
\usepackage{xcolor}

\newcommand{\cmt}[2]{}

\begin{document}

\title{Robust Probabilistic Model Checking\\with Continuous Reward Domains}

\author{
\IEEEauthorblockN{Xiaotong Ji\IEEEauthorrefmark{1}, Hanchun Wang\IEEEauthorrefmark{2}, Antonio Filieri\IEEEauthorrefmark{1}, Ilenia Epifani\IEEEauthorrefmark{3}} 

\IEEEauthorblockA{\IEEEauthorrefmark{1}\textit{Department of Computing, Imperial College London}\\}

\IEEEauthorblockA{\IEEEauthorrefmark{2}\textit{Department of Mathematics, Imperial College London}\\}

\IEEEauthorblockA{\IEEEauthorrefmark{3}\textit{Dipartimento di Matematica, Politecnico di Milano}\\}
}

\maketitle
\begin{abstract} 
Probabilistic model checking traditionally verifies properties on the expected value of a measure of interest. This restriction may fail to capture the quality of service of a significant proportion of a system's runs, especially when the probability distribution of the measure of interest is poorly represented by its expected value due to heavy-tail behaviors or multiple modalities.
Recent works inspired by distributional reinforcement learning use discrete histograms to approximate integer reward distribution, but they struggle with continuous reward space and present challenges in balancing accuracy and scalability.

We propose a novel method for handling both continuous and discrete reward distributions in Discrete Time Markov Chains using moment matching with Erlang mixtures. 
By analytically deriving higher-order moments through Moment Generating Functions, our method approximates the reward distribution with theoretically bounded error while preserving the statistical properties of the true distribution.
This detailed distributional insight enables the formulation and robust model checking of quality properties based on the entire reward distribution function, rather than restricting to its expected value.
We include a theoretical foundation ensuring bounded approximation errors, along with an experimental evaluation demonstrating our method's accuracy and scalability in practical model-checking problems.

\end{abstract}

\section{Introduction}\label{secIntro}

Probabilistic model checking (PMC)~\cite{2016KatoenPMCLandscape} plays a critical role in many self-adaptation methods, allowing the detection of quantitative requirements violations and supporting the planning of appropriate adaptation actions~\cite{2022PMCForAutonomousSystems,surveyFormalMethodsAdaptive,CACMAadaptiveNeedsQuantitative,7083754}.

However, most PMC methods allow to reason only about the expected value of a system property of interest, which may neglect the variability in the distribution of possible outcomes. For example, even if the average execution time of the system is less than 100ms, a large number of runs may exceed such expected value. Not accounting for the variability in a measure of interest's distribution may provide a partial assessment of a system's performance, in turn limiting the awareness of the adaptation logic and the robustness of its decisions.

Recent advancements in PMC~\cite{baier2014probabilistic, elsayed2024distributional}, also inspired by related results in reinforcement learning~\cite{bellemare2017distributional,dabney2018distributional} have incorporated distributional perspectives in the analysis of models abstracted as Discrete-Time Markov Chains (DTMCs) and Markov Decision Process (MDPs) by approximating probability or reward distributions using histograms. However, the precision of histogram approximations is constrained by the setting of bins, with fewer bins leading to coarser representations. Furthermore, the resulting piecewise constant approximations fail to capture subtle features such as multimodality or skewness, and their applicability is restricted to integer reward scheme to ensure convergence. These methods also face scalability challenges, particularly in large or continuous reward spaces, where sparse data worsens the inaccuracy of the discretization process.

\begin{figure}[!htp]
\centering
\subfigure[UAV's Flight Process DTMC]{\includegraphics[width=0.42\textwidth]{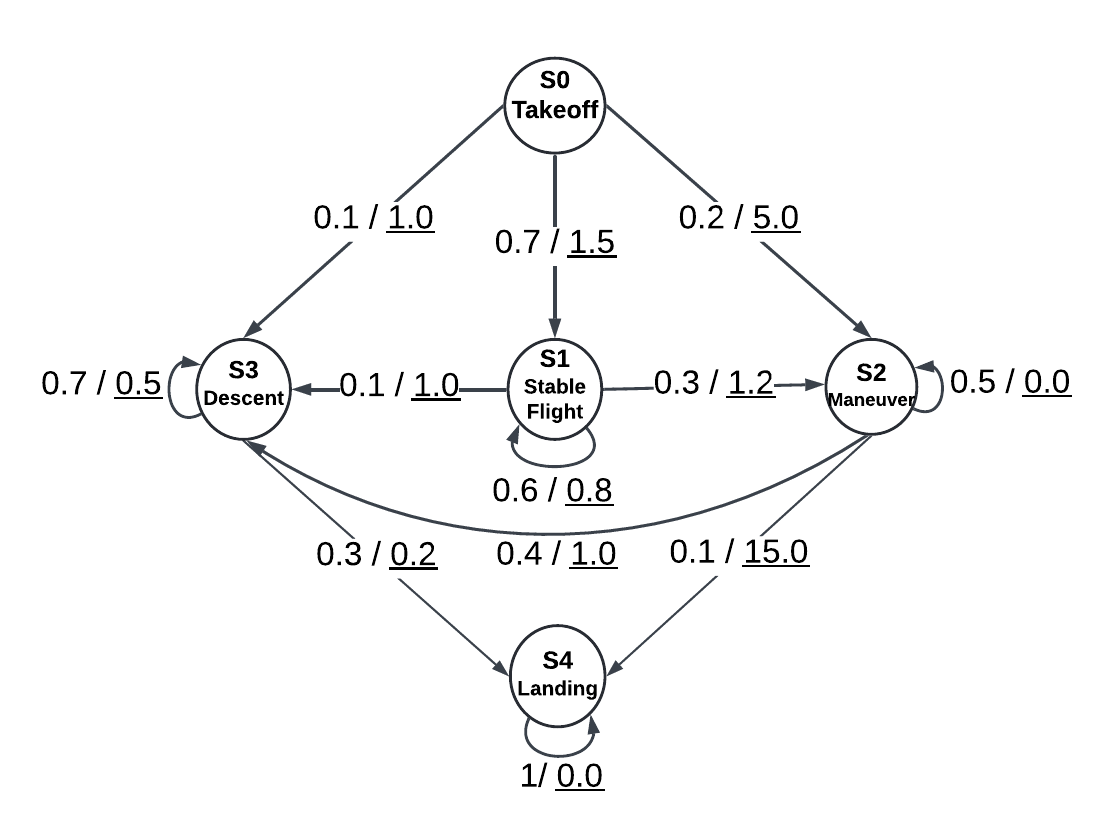}\label{figExampleDTMC}}
\subfigure[Cumulative reward Distribution from $s_0$]
{\includegraphics[width=0.42\textwidth]{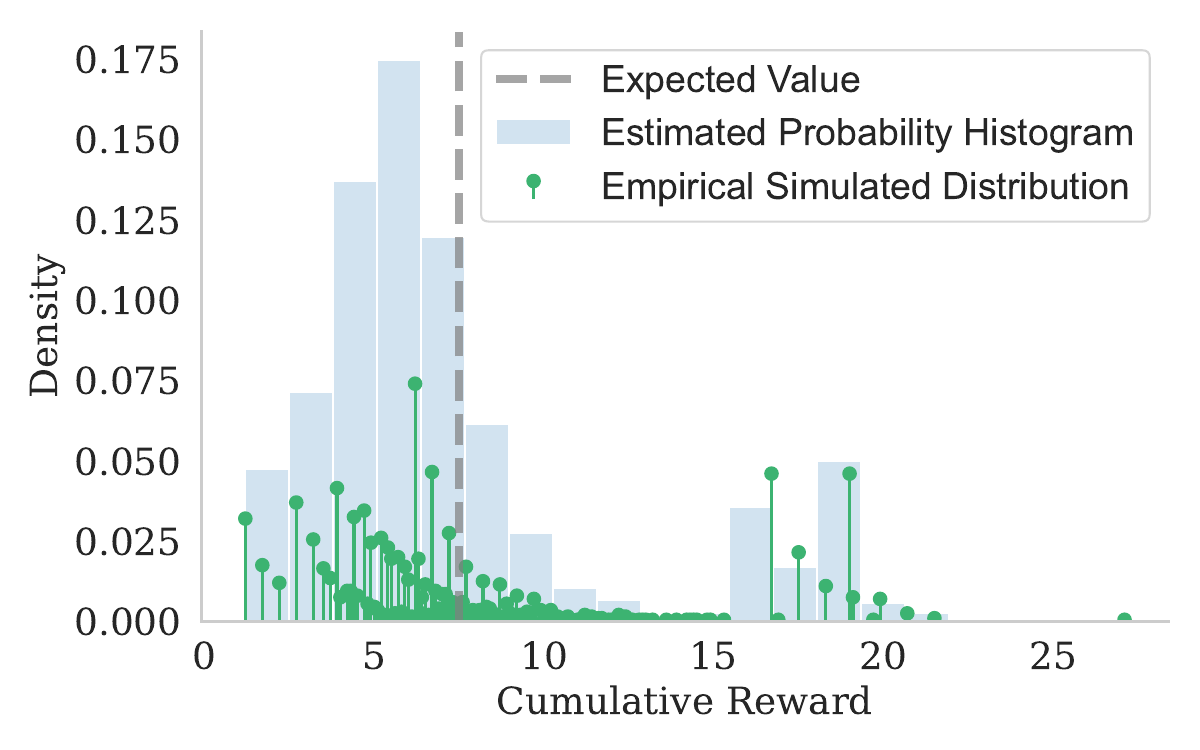}\label{figExampleDistribution}}
\caption{UAV flight process with states $\{s_0, s_1, s_2, s_3, s_4\}$ representing \texttt{Takeoff}, \texttt{Stable Flight}, \texttt{Maneuvering}, \texttt{Descent}, and \texttt{Landing}, respectively. Transition probabilities and rewards are shown on the edges in the form of $p$/$\underline{r}$. The cumulative reward distribution at $s_0$ is illustrated using the expected value (grey), probability histogram estimation (light blue), and true simulated distribution points (green), highlighting the different levels of information provided by various methods on the same distribution.}
\label{fig:example}
\end{figure}

\vspace{1mm}
\noindent\emph{Example.} Consider a UAV's flight process modeled as the DTMC shown in Figure~\ref{figExampleDTMC}, where each state corresponds to a specific flight phase -- Takeoff, Stable Flight, Maneuvering, Descent, and Landing. State transitions are governed by probabilities, while rewards represent energy consumption during these transitions. For example, transitioning from stable flight to maneuvering consumes more energy, whereas landing after descent requires minimal energy. This framework allows for evaluating the UAV's operational efficiency across various scenarios. Standard PMC only considers the expected value (the grey line in Figure~\ref{figExampleDistribution}) of the cumulative energy consumption, which provides limited information on the energy consumption pattern. In particular, it falls on the right-hand side of the head modality and far from the tail modality of the distribution. Recent categorical histogram-based methods~\cite{bellemare2017distributional} estimate the distribution of cumulative rewards, but they struggle to accurately capture the true distribution, especially in cases involving closely packed peaks in dense regions or isolated peaks in sparse regions. In dense regions, histograms tend to smooth out critical individual peaks, overestimating probabilities, while in sparse regions, they underestimate the significance of isolated high-reward transitions. Although quantile-based histogram methods~\cite{dabney2018distributional} allow for more flexible binning to address the dense-sparse issue, they fail to capture peak and tail behaviors effectively. These limitations highlight the need for more precise and flexible methods to approximate the true distribution of the measure of interest.
\vspace{1mm}

In this work, we propose a distribution approximation method using moment matching on Erlang mixtures to efficiently approximate the full distributional profile of probability and reward distribution over DTMCs, enabling verification of robust properties based on the complete distribution rather than expected values. The Erlang mixture is chosen for its strong capability to approximate all positive random variables with any arbitrary level of accuracy~\cite{schweitzer1996stochastic} and its natural connection to Markov models, facilitating a detailed analysis of Markov reward distributions to support the verification of quantitative requirements. The approximated distribution addresses the challenges of capturing multi-modality and precise behaviors faced by histogram approximations, without relying on iterative forward computations or simulations and supporting continuous reward spaces without requiring discretization, while preserving statistical properties, such as skewness and kurtosis, through precise moment matching. Our approach enhances the robustness of PMC by capturing a broader spectrum of possible outcomes, thus improving system behavior verification under uncertainty. 
Our key contributions are as follows:

\begin{enumerate}
    \item \textbf{Moment Matching Distribution Approximation}: We introduce a novel approach to approximate the cumulative reward distribution of stochastic systems using mixtures of Erlang distributions matched to its moments, providing a precise characterization of the full reward distribution with bounded error.
    
    \item \textbf{Robust Property Model Checking}: We introduce a procedure to verify \emph{chance requirements} of the form $Pr(R \leq r^{*}) \geq \alpha$ on a measure of interest, which predicate on the probability of the threshold $r^*$ being violated by a run of the system, rather than only constraining the measure's expected value.

\end{enumerate}

 \section{Background}\label{secBackground}

\subsection{Probabilistic Model Checking}
Probabilistic model checking is a formal method to reason about quantitative properties of systems, whose behavior is abstracted via convenient stochastic models~\cite{2016KatoenPMCLandscape}.
In this work, we focus on discrete Markov models and on the verification of regular properties~\cite{Baier2009MDPVerification}. Regular properties can be reduced to predicates on the probability of reaching a target state or on the cumulative reward until reaching a target state. For example, the co-safe fragment of linear time logic (LTL) is used to define sequential plans for a robot~\cite{pappas-cosafe-ltl} or probabilistic computation tree logic (PCTL) properties without nesting of the modal operators~\cite{Baier2009MDPVerification}. Without loss of generality, we will focus on properties of the forms $P_{\bowtie p}[ \diamond \phi]$ and $R_{\bowtie r}[ \diamond \phi]$, where $\bowtie \in \{\leq, \geq\}$, and $p \in [0, 1]$ and reward $r \in [0, \infty)$ represent bounds on the probability of eventually reaching a state where the Boolean predicate $\phi$ holds. 

In its standard semantics~\cite{baier2008principles}, a formula $P_{\bowtie p}[ \diamond \phi]$ holds in a state $s$ if the \textit{expected probability} of eventually reaching a state satisfying $\phi$ is $\bowtie p$. Similarly, $R_{\bowtie r}[ \diamond \phi]$ is evaluated by constraining the expected cumulative reward until a state satisfying $\phi$ is reached. This formulation takes advantage of the linearity of the expected value to reduce the probabilistic model checking problem to the solution of a linear system of equations for DTMCs, while for MDPs~\cite{puterman1990markov} to efficient dynamic programming using the standard Bellman operator~\cite{Baier2009MDPVerification}. 

However, the expected value may not tell the full story about the system being analyzed~\cite{elsayed2024distributional}. 
For example, the dispersion of the distribution, its skewness, or the presence of multiple modalities may render the expected value a poor representative of the system's behavior. 

\vspace{1.5mm}
\noindent\emph{Example.} For the UAV system in Figure~\ref{fig:example}, the expected cumulative reward is 7.51, which is rarely achieved across runs and does not capture the bi-modal behavior of the system.

\vspace{1mm}
\noindent\textbf{Chance constraints.} 
In this work, we aim to reinterpreting $P_{\bowtie p}(\cdot)$ and $R_{\bowtie r}(\cdot)$ as evaluations of the cumulative distribution function for the reachability and reward distributions. This enables verifying requirements of the form ``a reward larger than $r$ is obtained with at least probability $c$''. Notice that 1) reasoning about the probability of reaching a target state (i.e., a state satisfying a desired property $\phi$) can be reduced to reasoning about the cumulative rewards to reaching a target state via systematic model transformations~\cite{2011MultiObjectivePMC} and, similarly, 2) the paths reaching a target state can be reduced to paths to absorption into a designated absorbing state, again via the identification of connected components and systematic model transformations~\cite{Baier2009MDPVerification}. Therefore, to simplify the presentation and without loss of generality, the remainder of this paper will focus on approximating and reasoning about the probability distribution of the cumulative reward to absorption. More precisely, we adapt the definition of chance-constrained requirements from the robust control literature~\cite{2017ChanceConstraints} and propose an approximate decision procedure, with bounded error, for requirements of the form 
$Pr(X \leq r^*) \geq \alpha$
, where $X$ represents the distribution of cumulative rewards to absorption over the possible runs of the Markov process, $r^*$ is a threshold value, and $\alpha$ is the required probability level. Notably, chance-constrained requirements are also more general than the commonly used (Conditional)Value at Risk requirements used in previous literature~\cite{elsayed2024distributional}.

\subsection{Markov Reward Process}\label{secMarkovRewardProcess}

\noindent\textbf{Discrete-Time Markov Chain (DTMC).} 
A Discrete-Time Markov Chain is defined as a tuple $D = (S, s_0, P, r, AP, L)$, where $S$ is a finite state space, $s_0 \in S$ is the initial state, $P\colon  S \times S \to [0,1]$ is the stochastic matrix representing the transition probabilities (such that $p_{ij} \geq 0$ for all $p_{ij} \in P$ and $\sum_j p_{ij} = 1$ for all $i$), and $r\colon  S \to \mathbb{R}_{\geq 0}$ is a reward function that assigns a reward to each state\footnote{transition rewards are sometimes used to make model representations more compact (e.g., Fig.~\ref{fig:example}) but they can be reduced to state rewards w.l.o.g. by systematically replacing a transition $(s, p/r, s^\prime)$ with two transitions $(s, p, q)$ and $(q, 1, s^\prime)$ where $q$ is a fresh state with associated state reward $r$.}, $AP$ is a set of atomic propositions and $L\colon S \to 2^{AP}$ is a labeling function mapping each state to a set of propositions that hold true in that state. This stochastic system evolves as a sequence of random variables $\{X_n\}_{n \ge 0}$, where each $X_n$ represents the state of the system at the $n$-th time step. 

\noindent Notably, a specific \emph{induced DTMC} $M_D = (S_D, s_0, P_D, r, AP,$ $L)$ can be constructed by a Markov Decision Process (MDP) $M$ and a given policy $\pi$, where $S_D$ is a sub-space of the finite state space of the original MDP, $s_0 \in S_D$ is the initial state,  $P_D\colon S \times A \times S \to [0,1]$ is the stochastic transition matrix and $r\colon  (S, A) \to \mathbb{R}_{\geq 0}$ is a reward function. The policy $\pi\colon  S \times A \to \{0, 1\}$ is a specific action $a$, $\pi(a'|s) = 1$ and $\pi(a'|s) = 0$ for all $a' \ne a$. $M_D$ is fully probabilistic with no open action choices.

\vspace{0.5mm}
\noindent\textbf{Reward Process.} Consider a non-negative reward function \( r\colon S \to \mathbb{R}_{\geq 0} \), representing the instantaneous reward received in each state, together with a sequence of random variable $\{X_n\}_{n \ge 0}$. The total cumulative reward up to the time $N$ over the times $0, \dots, N$ is $R(N) = \sum_{n=0}^{N} r(X_n)$. As the expectation of probabilities and rewards in Markov Models can be computed as the solutions of linear equations~\cite{cinlar2013introduction}, the expected total cumulative reward $\mu = \mathbb{E}[ \sum_{n=0}^{T} r(X_n)]$ can be easily computed by solving the linear system $\mu = r + P\mu$ if the system is transient. Furthermore, by utilizing the Moment Generating Function (MGF) of total reward to absorption, we can compute the $k$-th moments of total reward for higher moment analysis.

\vspace{1mm}
\noindent\textbf{Moment-Generating Function (MGF).}  Consider a set of transient states \( C \subseteq S \) and absorbing states \( C^c \subseteq S \) (where $C \cap C^c = \emptyset$, $C \cup C^c = S$), with \( T = \inf \{ n \geq 0\colon X_n \in C^c \} \) denoting the first hitting time of a state in \( C^c \). The total cumulative reward up to absorption is represented by \( R_T = \sum_{n=0}^{T-1} r(X_n) \). The  MGF of \( R_T \) starting from $X_0=x$, denoted by \( M_{R_T}(\theta, x) \), is given by:
\begin{equation}
    M_{R_T}(\theta,x) =\mathbb{E}_x[e^{\theta R_T}] = \mathbb{E}_x\left[\exp\left(\theta \sum_{n=0}^{T-1} r(X_n)\right)\right]    
\end{equation}
Using first-step analysis, this MGF $M_{R_T}$ satisfies the linear system:
$$
M_{R_T}(\theta) = f(\theta) + G(\theta) M(\theta),
$$
where \( f(\theta, x) = \sum_{y \in C^c} e^{\theta r(x) } P(x, y) \) for \( x \in C \), and $ G(\theta, x, $ $ y) = e^{\theta r(x)} P(x, y) $ for \( x, y \in C \). 
The $k$-th moment $u_{k}(x)=\mathbb{E}_x\left[R_T^k\right]$ starting at $X_0=x$ can be obtained by differentiating the MGF $u_k=\left.\frac{d^k}{d \theta^k} M_{R_T}(\theta)\right|_{\theta=0}$ with the following recurrence:

\begin{equation}\label{eqMGF}
    \begin{split}
    u_{k}(x) &= \sum_{y \in S} P(x, y) \mathbb{E}_y \left[ \sum_{n=0}^{T-1} r(X_n) \right]^k \\
&= r(x)^k+\sum_{i=0}^{k-1}\binom{k}{i} r^{k-i}(x) \sum_{y \in S} P(x, y) u_i(y)
\end{split}
\end{equation}

\vspace{1.5mm}
\noindent\emph{Example.} The first moment for our UAV computed with Equation~\eqref{eqMGF} corresponds to the expected reward computed, e.g., by the PRISM model checker~\cite{kwiatkowska2002prism}: $\mu = 7.51$. However, the standard deviation (square root of the second moment) is $\sigma = 4.91$. Using, for example, Cantelli's inequality~\cite{CantelliInequality} (as known as one side Chebysehv's inequality~\cite{saw1984chebyshev}), this implies only $50\%$ of the trajectories are guaranteed to fall within the interval $[0, \mu + \sigma]$, leaving a significant portion of distribution not represented.

This result indicates that the expected value alone provides limited information, highlighting further refinement for accurate distribution approximations and robust model checking, particularly in risk-sensitive systems, to ensure reliable outcomes. We introduce next the base distribution we will use for such approximations.

\vspace{1mm}
\noindent \textbf{Erlang Distribution.}  
The Erlang distribution is a special case of both phase-type and Gamma distributions with integer shape parameters~\cite{erlang1909theory}. An Erlang distribution with shape parameter $a \in \mathbb{Z^+}$ and rate parameter $\lambda \in \mathbb{R}^+$ is defined by the probability density function:
\begin{equation}
    f_X(x; a, \lambda) = \frac{\lambda^a x^{a-1} e^{-\lambda x}}{(a-1)!}
\end{equation}

The Erlang cumulative distribution function is given by:

\begin{equation}\label{eqErlangMixCDF}
F_X(x; a, \lambda) = 1 - \sum_{j=0}^{a-1} \frac{(\lambda x)^j e^{-\lambda x}}{j!}
\end{equation}

To approximate the cumulative reward to absorption distribution, we will employ a mixture of Erlang distributions, following its inherent connection with Markov models~\cite{10.1214/21-BA1272} and its strong ability to closely approximate the distributions of any random variable in continuous settings and discrete settings with very fine granularity~\cite{doi:10.1080/15326349208807217,schweitzer1996stochastic,he2022continuous}. 
In particular, we will construct a mixture of Erlang distributions with a common rate parameter, which reduces the cost of moment matching without loss of approximation accuracy~\cite{chakravarthy1996matrix}. Such functional form allows to recover the full distributional behavior of the reward process with any arbitrary accuracy level, matching the first $K$ moments for any $K>0$.

 \algrenewcommand\algorithmicrequire{\textbf{Input:}}
\algrenewcommand\algorithmicensure{\textbf{Output:}}
\section{Robust Probabilistic Model Checking}\label{secMethod}
We propose a robust PMC method for DTMCs by approximating the cumulative reward distribution towards absorption through moment matching with mixtures of Erlang distributions. The approximation is formulated as a truncated Stieltjes moment problem~\cite{gavriliadis2012truncated}, to reconstruct the density function from a finite number of its moments on the semi-infinite interval $[0, +\infty)$. This approach specifically focuses on preserving the statistical properties of the distribution while addressing the challenge of having only partial information in the first $K$ moments, rather than the full infinite moment sequence. 
This approach avoids the complexities of discretization and iterative updates of histogram-based methods limited to discrete rewards~\cite{elsayed2024distributional}, thus avoiding potential issues of binning resolution trade-off and state-space explosion. Instead, our method focuses on refining the approximation based on a set of invariant statistical characteristics to approximate the density function with essential statistical information directly.

\subsection{Markov Reward Distribution Approximation}\label{secDistributionApproximation}

\noindent\textbf{Problem Formulation.} Our objective is to approximate the distribution of the cumulative reward towards absorption $R(N)$ in a DTMC \(D\), with initial state \(s_0\) and transition matrix \(P\). We define the approximating density \(f_{\text{approx}}(x)\) of the approximate reward distribution $X$ as a mixture of $n$ Erlang densities:
\begin{equation}
    f_{\text{approx}}(x) = \sum_{i=1}^n \omega_i f_{\text{Erlang}}(x; a_i, \lambda),
\end{equation}
\noindent where \( \omega_i \geq 0 \) is the non-negative mixture weight of the $i$-th Erlang density, \( a_i \in \mathbb{Z}^+\) is the shape parameter of the $i$-th Erlang density, and \( \lambda \in \mathbb{R}^+ \) is the common rate parameter.

We formalize the approximation process as solving a truncated Stieltjes moment problem~\cite{gavriliadis2012truncated} on the semi-infinite interval $[0, +\infty)$ with the first $K$-th moments of the random variable $X$. The parameters \( \omega_i, a_i, \lambda \) are determined by solving an optimization problem that minimizes the difference between the first $K$-th moments $\mu_1, \ldots\mu_K$ of the true reward distribution (cf. Sec.~\ref{secMarkovRewardProcess}) and the moments $\hat{\mu}_1, \ldots, \hat{\mu}_K$ of the approximate distribution. We further incorporate the maximum entropy principle within our moment matching formulation~\cite{mead1984maximum, ganapathi2012constrained}, to ensure that among all possible approximate distributions matching the given moments, the one with the least bias (i.e., the one maximizing entropy) is selected. 

The objective function of the optimization problem is defined as:
\begin{equation}
\begin{aligned}
\min_{\omega_i, a_i, \lambda} \quad & \mathcal{L} = \sum_{k=1}^K \left( \mu_k - \hat{\mu}_k \right)^2 - \gamma H(f_{\text{approx}}), \\
\end{aligned}
\end{equation}

\noindent where \( \gamma \geq 0 \) is a weighting parameter controlling the trade-off between moment matching and entropy maximization, 

\noindent\( H(f_{\text{approx}}) = -\int_{0}^{\infty} f_{\text{approx}}(x) \log f_{\text{approx}}(x) \, dx \) is the differential entropy of the approximated distribution, the moments $\mu_k$ of the true reward distribution are computed with the MGF in Equation~\eqref{eqMGF}, and the moments of the mixture of Erlang distributions is $\hat{\mu}_k = \sum_{i=1}^n \omega_i \frac{(a_i + k - 1)!}{(a_i - 1)!} \frac{1}{\lambda^k}$.

\begin{algorithm*}[t]
\caption{Markov Reward Distribution Approximation}
\label{alg:alg1}
\begin{algorithmic}[1]
\Require DTMC \(D\) (or induced DTMC \(M_D\)) $= (S, s_0, P, r, AP, L)$, error bound \(\epsilon > 0\), the number $K$ of moments to be matched, and number $n$ of Erlang distributions to mix
\Ensure Approximated mixture of \(n\) Erlang distributions \(f_{\text{approx}}(x)\)

\State $f = \sum_{y \in T} e^{k r(x)} P(x, y)$ $x \in S$ and $G = e^{k r(x)} P(x, y)$ $x, y \in S$\label{lnMoments1}
\State $\mu_k = r^k f + \sum_{i=0}^{k} \binom{k}{i} G^{(k-i)}$ \quad $\forall k \in \{1, \dots, K\}$\label{lnMoments2}
\Comment{Compute analytical moments with Eq.~\eqref{eqMGF}}

\State Initialize \( \mathcal{L}=\infty \) and \(\omega_i, a_i, \lambda \) \quad \(\forall i \in \{1, \dots, n\} \)

\Repeat\label{lnOptimizationLoop}
    \State $\mathcal{L}_{\text{prev}}$ = $\mathcal{L}$
    \State $\hat{\mu}_k = \sum_{i=1}^{n} \omega_i \frac{(a_i + k - 1)!}{(a_i - 1)!} \frac{1}{\lambda^k}$ \quad $\forall k \in \{1, \dots, K\}$\label{lnPolyForm}
    \Comment{Compute approximate moments}
    
    \State $H(f_{\text{approx}}) = - \int_{0}^{\infty} f_{\text{approx}}(x) \log f_{\text{approx}}(x) \, dx$
    \Comment{Calculate the differential entropy}
    
    \State $ \min_{\omega_i, a_i, \lambda} \mathcal{L} = \sum_{k=1}^{K} \left( \mu_k - \hat{\mu}_k \right)^2 - \gamma H(f_{\text{approx}})$\label{lnOptimizationProblem}
    \\\quad\quad\quad subject to $\sum_{i=1}^n \omega_i = 1, \quad \omega_i \geq 0, \quad a_i \in \mathbb{N}^+, \quad \lambda > 0$

    \State $\omega_i, a_i, \lambda \gets \arg\min_{\omega_i, a_i, \lambda} \mathcal{L}$
    \Comment{Update parameters}

\Until \( \left| \mathcal{L}_{\text{prev}} - \mathcal{L} \right| < \epsilon \)

\State \Return \( f_{\text{approx}}(x) = \sum_{i=1}^n \omega_i f_{\text{Erlang}}(x; a_i, \lambda) \)

\end{algorithmic}
\end{algorithm*}
\vspace{1mm}
\noindent \textbf{Main Algorithm.} The approximation algorithm, detailed in Algorithm~\ref{alg:alg1}, takes as input a DTMC \(D\) along with an arbitrarily small error bound \(\epsilon > 0\), the number $K$ of moments to be matched, and the size of the mixture $n$, and it returns the approximate distribution. For a given number $K$ of moments, the mixture of Erlang distributions can match them using approximately $\lfloor K/2 \rfloor + 1$ components, reducing the mixture size while preserving accuracy~\cite{johnson1989matching}. We first compute the true target moments $\mu_k$, providing the necessary statistical information using Equation~\eqref{eqMGF} (lines~\ref{lnMoments1}-\ref{lnMoments2}). Before the optimization phase, we initialize the mixture weights $\omega_i$, shapes $a_i$, and rate parameter $\lambda$ (detailed discussion on initialization and hyper-parameters in Section~\ref{SecEval}).

At each optimisation step (in the loop at line~\ref{lnOptimizationLoop}), we compute the first \(K\) moments of \(f_{\texttt{approx}}\) and the differential entropy \(H(f_{\texttt{approx}})\). The objective function \(L\) balances moment matching with entropy maximization, controlled by the weighting parameter \(\gamma\). The parameters \(\omega_i\), \(a_i\), and \(\lambda\) are updated to minimize \(L\) subject to the constraints, until the convergence criteria are satisfied. The algorithm then returns the approximated distribution with the optimized parameters for further analysis via robust model checking. To improve the numerical stability of moment matching, we further utilize standardized moments $\bar{\mu}_{k}$ in the optimization process 
$\bar{\mu}_{k}= \mu_{k} / c$, with $c=\sqrt{\mu_{2}}$, for $k=3, 4, \dots, K$.

Algorithm~\ref{alg:alg1} produces an arbitrarily close approximation $f_{\texttt{approx}}$ that closely matches the cumulative reward distribution associated with discrete Markov models. Arbitrary accuracy can be obtained by increasing the number of moments to be matched and the size of the mixture~\cite{gavriliadis2012truncated,schweitzer1996stochastic}. The inclusion of entropy maximization ensures that, among all distributions matching the given moments, the algorithm selects the most unbiased one. This feature is particularly advantageous for robust probabilistic model checking to avoid overfitting on a small number of moments~\cite{mead1984maximum}.

\vspace{1mm}
\noindent\textbf{Complexity Observations.}
The main algorithm starts with a preliminary phase which computes the first $K$ moments of the reward distribution based on the reward scheme and transition matrix of the DTMC. Each moment $\mu_k$ is computed by solving a system of linear equations involving the lower-order moments up to $\mu_{k-1}$. Therefore, to compute the first $K$ moments, it is sufficient to solve $K$ systems of linear equations. The size $\xi$ of each system of equations corresponds to the size of the transition matrix $P$, i.e., the number of states in the model. A worst-case complexity for solving a system of linear equations is $\mathcal{O}(\xi^3)$, although standard numerical solvers (e.g., Numpy~\cite{harris2020array}) usually improve over such a bound and can exploit the sparsity and possible symmetries in the transition matrix for more efficient computations. 

In the optimization loop, we separate the computation of the differential entropy of $f_{\texttt{approx}}$ from the main optimization problem. While this step may increase the number of iterations to convergence, it reduces the complexity of the inner optimization problem at Line~\ref{lnOptimizationProblem}. This optimization problem is a mixed-integer non-linear form (the parameters $a_i$ being positive integers), which places it in the NP class, which may be computationally intractable. 

A common heuristic adopted in the literature to reduce the complexity of optimization is to fix a set of possible values for the shape parameters $a_i$~\cite{Verbelen_Gong_Antonio_Badescu_Lin_2015,Lee_Lin_2012} while the optimizer identifies the optimal rate parameter and the weights of the mixture. While setting each $a_i=i$, for $i=1, 2, 3, \dots$ ensures an optimal approximation by relying on the density of Erlang mixtures~\cite[Theorem 2.9.1]{schweitzer1996stochastic}, it may also require a large number of components to approximate the underlying distribution, especially in presence of heavy tails, sharp drops, or multiple modalities. Furthermore, using a large mixture to fit a small number of moments may lead to overfitting. A common recommendation when fitting a finite number of moments $K$ is instead to spread the shape parameters~\cite[Sec 5.4.3]{verbelen2013phase} fed to the optimizer. We found in our experiments that an exponential spreading (i.e., $a_i=c^i$) of the fixed shapes yielded the best results, compared to a dense allocation (i.e., $a_i=i$) and the linear spreading (i.e., $a_i=c\cdot i$ for some $c>1$) used in~\cite{verbelen2013phase}. Fixing the shape parameters reduces the optimization problem to a non-convex quadratic form~\cite{2012MIQProgramming} that state-of-the-art optimizers -- such as Gurobi~\cite{gurobi} or SciPy~\cite{virtanen2020scipy} -- can solve quickly, as demonstrated experimentally in Section~\ref{SecEval} where we also report on the accuracy and computation time required for different spreading strategies and number of components.

Finally, in the next section we will describe how probability inequalities can be used to draw sufficient decisions on chance-constrained requirements directly from the moments computed during the preliminary phase of Algorithm~\ref{alg:alg1}, thus avoiding the need to execute the optimization loop entirely.

\subsection{Distributional Analysis for Robust Model Checking}

\noindent \textbf{Robust Property Evaluation.} The evolution of the random variable \(X\) can be treated as the outcome cumulative reward of the underlying stochastic system, with each transition governed by probabilistic dynamics. By considering the density \(f_{\texttt{approx}}\) of \(X\), we formalize a robust property evaluation for chance-constrained requirements~\cite{2017ChanceConstraints} of the form:
\begin{equation}\label{eqChanceConstrained}
Pr(X \leq r^*) \geq \alpha,    
\end{equation}
\noindent where \(r^*\) is a critical threshold representing a performance or safety metric, and \(\alpha\) is the required probability level. 
For example, in a safety-critical system, we can verify that the probability of a performance violation remains below a predefined risk level \(\alpha\), rather than asking whether just the expected reward is below the threshold. 
Notice that constraints of the form $X \geq r^*$ or, more in general, $r_* \leq X \leq r^*$, can be evaluated with the same method using the standard properties of the cumulative distribution function.

Almost sure properties, where the probability of satisfying a temporal formula is required to be 1 (or 0), can also be straightforwardly phrased as chance-constrained requirements, whereas chance-constrained formulations allow for the controlled relaxation of the requirements. This makes our distributional method well suited for the verification of probabilistic safety requirements~\cite{2016KatoenPMCLandscape}, as well as for the robustness analysis of the policies synthesized for MDPs via reinforcement learning~\cite{gross2022cool}, by reasoning on the induced DTMC model~\cite[Ch. 10]{baier2008principles}.

From $f_{\texttt{approx}}$, the evaluation of the chance-constrained requirement in Equation~\eqref{eqChanceConstrained} reduces to computing the cumulative distribution function of the finite Erlang mixture $X$:
\begin{equation}\label{eqMixtureCDF}
    Pr(X \leq r^*) = \sum_{\omega_i} \omega_i F_{X_i}(r^*; a_i, \lambda_i)=: F_{\texttt{approx}} (r^*) 
\end{equation}

\noindent where the cumulative distribution functions $F_{X_i}(x; a_i, \lambda_i)$ are defined for each Erlang component of the mixture as in Equation~\eqref{eqErlangMixCDF}.

\vspace{1mm}
\noindent \textbf{Efficient Sufficient Conditions.} 
After the moments of the rewards distribution are computed in Algorithm~\ref{alg:alg1} (lines~\ref{lnMoments1} and ~\ref{lnMoments2}), a decision on the validity of a chance-requirement may already be possible, without computing $f_{\texttt{approx}}$, by relying on probabilistic inequality derived from Markov's theorem~\cite{CantelliInequality}. Such inequalities can bound the probability of a random variable exceeding a threshold by computing only with a finite number of moments of the distribution. Being agnostic to the actual shape of the distribution, such inequalities tend to be over-conservative, as we will show later.

Nevertheless, probability inequalities, when they confirm the satisfaction of the chance-constrained requirements, can be used as an early termination criterion for Algorithm~\ref{alg:alg1}, allowing to skip the optimization loop entirely. Under our assumption of non-negative rewards, the sharpest inequality derived from Markov's theorem is Cantelli's~\cite {CantelliInequality} (also referred to as one-sided Chebyshev's inequality). Its formulation generalized to higher moments is:
\begin{equation}\label{eqCantelliHigherMoments}
    Pr(X - E[X] \geq a) \leq E \Biggl[ \Biggl(\frac{X + b}{a + b}\Biggr)^n \Biggr] = \frac{E[ (X + b)^n ]}{(a + b)^n}
\end{equation}

\noindent where $E[ (X + b)^n ]$ is the $n$-th moment of $X$ centered in $-b$. When $n=2$, Equation~\eqref{eqCantelliHigherMoments} reduces to the original formulation of Cantelli's inequality by choosing $b=\sigma^2/a$, $\sigma$ being the standard deviation of $X$ (the low-tail correspondent, can be obtained by applying the inequality to $-X$, thus flipping the constraint to a $\leq$ form).

After computing the first $K$ moments of the reward distribution in the initialization phase of Algorithm~\ref{alg:alg1}, Equation~\eqref{eqCantelliHigherMoments} can provide a fast and robust decision on the probabilistic model checking problem by comparing its right-hand side with $1-\alpha$. If the chance constraint holds, the decision procedure terminates. Otherwise, it is still possible that Cantelli's inequality is not satisfied due to its conservativeness -- the inequality must hold for \emph{any} distribution. In such a case, reconstructing $f_{\texttt{approx}}$ is required to obtain a more accurate answer.

\vspace{3mm}
\noindent\emph{Example.} 
Figure~\ref{fig:cantelli} shows the cumulative distribution for our UAV running example. The solid green line shows the CDF $F_{\texttt{approx}}$ (integral of $f_{\texttt{approx}}$) estimated with our method by matching the first three moments to construct a mixture of three Erlang distributions. The shaded bars show the empirical distribution drawn from one million simulations of the system, which we use to approximate the ground truth. The dashed vertical lines show the probability bounds computed using Cantelli's inequality for the chance requirement of Equation~\eqref{eqChanceConstrained}, with values of $r^*$ being on the x-axis and each dashed line showing a different probability threshold $\alpha$.
Consider the requirement $Pr(X \leq 15) \geq 0.65$. Cantelli's inequality ($n=3$) ensures that $Pr(X \leq 15) \geq 0.68$, therefore the requirement holds and there is no need to compute $f_{\texttt{approx}}$ (optimization loop in Algorithm~\ref{alg:alg1}) to draw a conclusion. If we were instead to verify $Pr(X \leq 15) \geq 0.8$, due to its conservative nature, Cantelli's bound would not be able to prove that the system indeed satisfies the requirement, as correctly concluded by evaluating $F_{\texttt{approx}}(15)$.

\begin{figure}[!htp]
\centering
\includegraphics[width=0.48\textwidth]{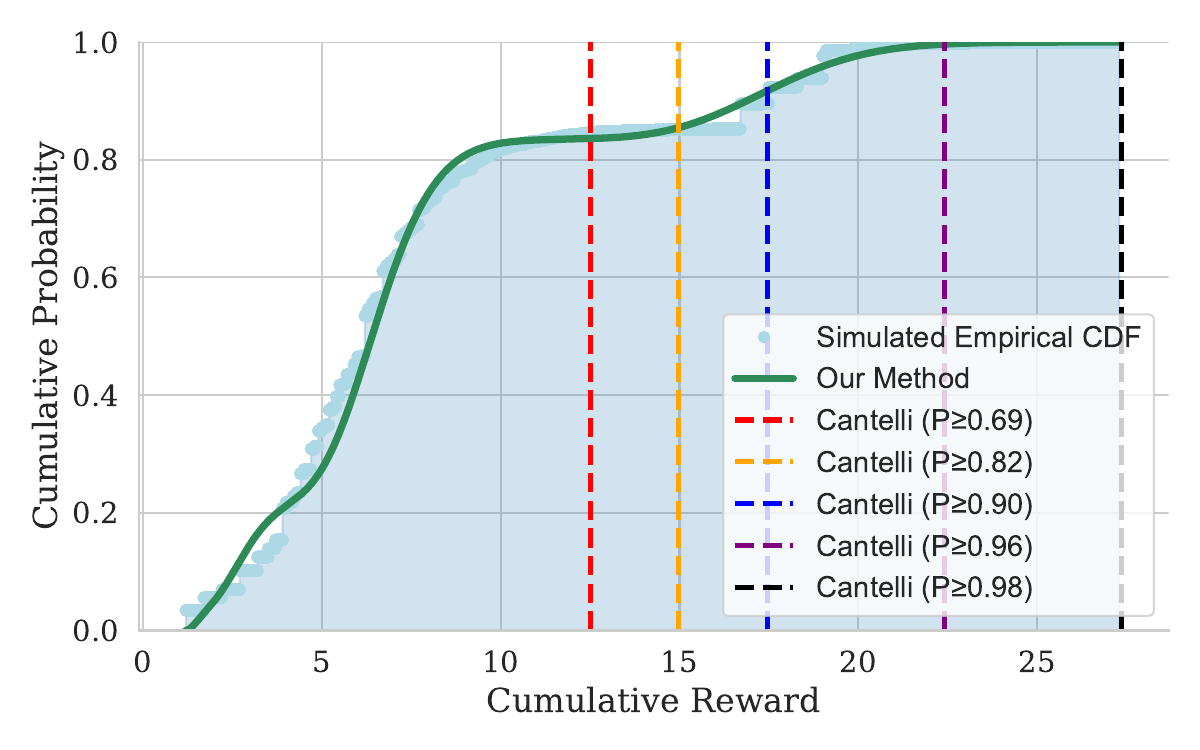}
\caption{Cantelli Bound's Coarseness in Model Checking
}
\label{fig:cantelli}
\end{figure}

\noindent The moment matching process ensures that the approximation closely represents the true stochastic behavior of the system using only the first $K$ moments \cite{gavriliadis2012truncated}. The use of multiple moments, captures both central tendencies and tail behaviors, which are essential for assessing system robustness. The next section includes the main theoretical foundation underpinning our choice of approximants and moment-matching, while in Section~\ref{SecEval}, we assess experimentally the quality of the approximants with different numbers of moments and mixture size, as well as the computation time required to construct them.

 \section{Theoretical Analysis}\label{secTheo}

In this section, we demonstrate that it is acceptable and effective to approximate the cumulative reward distribution of a DTMC using mixtures of Erlang distributions. We use Erlangization to obtain a continuous approximation which is later approximated by a mixture of Erlang distribution on a modified DTMC with discretized reward. By applying these theorems, we justify the use of mixtures of Erlang distributions as a valid method for approximating the cumulative reward distribution and its exponential tail behavior.

\vspace{1mm}
\begin{theorem}[Approximation of Cumulative Reward by Discrete Phase-Type (DPH) Distribution]
\noindent Consider a DTMC $D = (S, s_0, P, \mathbf{r})$ with a finite state space and non-negative rewards 
$r_i \in \mathbb{R}_{\ge 0}$ and a cumulative reward towards absorption $F(r) = R_T(s_0)$. 
For any  \( \epsilon > 0 \), there exists a precision parameter $\delta>0$ and a DTMC
$\tilde{D} = (\tilde{S}, s_0, \tilde{P}, \delta\textbf{e})$, where $\textbf{e}$ is the vector of all ones, such that $F(r)$ can be approximated by the DPH distribution $F_{\tilde{D}}(r) = \operatorname{DPH}(s_0, \tilde{P}) (r/\delta)$, in such a way that 
$$\sup _{r \in\left[0, +\infty\right)}\left|F(r)-F_{\tilde{D}}(r)\right| \leq \epsilon . $$
\end{theorem}
\begin{proof}
\noindent For any given precision $ \delta > 0 $, we construct an extended DTMC $\tilde{D} = (\tilde{S}, \tilde{P}, \tilde{\mathbf{r}})$ by discretizing the rewards in $D$. Specifically, for each reward $r_i \in \mathbf{r}$, define the approximate integer reward $ k_i = \left\lceil {r_i}/{\delta} \right\rceil $, where $k_i \delta$ is the closest integer multiple of $\delta$ that approximates $r_i$. The extended reward vector $\tilde{\mathbf{r}} = \{\delta, \ldots, \delta\} = \delta \mathbf{e}$ assigns a constant reward of $\delta$ to each state.
For each state $s_i \in S$ with approximate reward $ k_i \delta$, we introduce $n_i = k_i - 1$ additional intermediate states, denoted as $ s_{i,1}, s_{i,2}, \ldots, s_{i,n_i}$. The extended state space becomes:
\[\tilde{S} = \{ s_i, s_{i,1}, \ldots, s_{i,n_i} \colon s_i \in S \}.\]
The extended transition matrix $\tilde{P} \in \mathbb{R}^{|\tilde{S}| \times |\tilde{S}|}$ is defined as follows: for each transition $s_i \to s_j$ in the original matrix $P$, we introduce intermediate transitions $s_i \to s_{i,1} \to s_{i,2} \to \cdots \to s_{i,n_i} \to s_j$. Formally, the extended transition matrix $\tilde{P}$ replaces each $p_{ij}$ by $ \tilde{P}_{ij} $, defined as:
\[
\tilde{P}_{ij} = \left(\begin{array}{cccc}
0 & 1 &  \cdots & 0 \\
0 & 0 &  \cdots & 0 \\
\vdots & \vdots & \ddots & 1 \\
p_{ij} & 0 & \cdots & 0
\end{array}\right)_{k_i \times k_i} 
= \left(\begin{array}{cc}
0 & \mathbb{I}_{n_i \times n_i} \\
p_{ij} & 0
\end{array}\right).
\]
The maximum error per reward is bounded by $0 \leq k_i \delta - r_i < \delta$. Therefore, the cumulative reward error after the termination time $T$ satisfies $0 \leq \tilde{R}_T(s_0) - R_T(s_0) \leq T \delta$. Applying Markov's inequality yields:
\[
Pr\left( |\tilde{R}_T(s_0) - R_T(s_0)| \geq \epsilon \right) \leq Pr\left( T \geq \frac{\epsilon}{\delta} \right) \leq \frac{E[T] \delta}{\epsilon}.
\]
As $\delta \to 0$, the right-hand side approaches zero, leading to the weak convergence of the cumulative distribution functions: $\lim_{\delta \to 0} F_{\tilde{D}}(r) = F(r)$  for any continuity point $r$ of $F$.
\end{proof}

Thus, the extended DTMC \( \tilde{D} \), constructed by introducing intermediate states and modifying the transition matrix \( P \) to accumulate integer rewards in steps of \( \epsilon \), approximates the cumulative reward distribution of the original DTMC $D$. The resulting transition matrix \( \tilde{P} \) provides a DPH representation of the reward distribution, enabling further approximation by other tractable distributions, such as Erlang mixtures, while ensuring accurate control over the approximation error.

\vspace{1mm}
\begin{theorem}[Continuous Approximation of DPH by Erlangization]
\label{thm: dph_to_ph}
Consider $X$ be a DPH random variable with representation $\operatorname{DPH}(\boldsymbol{\alpha}, P)$. Define the continuous phase-type random variable $X_m$ by $\operatorname{PH}_c\left(\boldsymbol{\gamma}_m, T_m\right)$, where $\boldsymbol{\gamma}_m=\boldsymbol{\alpha} \otimes \mathbf{e}(1)$ and $T_m=\lambda\left(I \otimes J_0+P \otimes J_1\right)$ where
$$
J_0=\left(\begin{array}{cccc}
-1 & 1 & & \\
 & \ddots & \ddots & \\
 & & -1 & 1 \\
 & & & -1
\end{array}\right), \, J_1=\left(\begin{array}{llll}
0 & & &  \\
& 0 & &  \\
& & \ddots & \\
1 & 0 & \ldots & 0
\end{array}\right)$$
such that by replacing the constant sojourn time in Markov chain of $X$ with an Erlang random variable $E_{\lambda, m}$, where $\lambda = m$. Then $X_m$ is a continuous approximation of $X$ and converges to $X$ in distribution and moments, such that:
$$
\mathbb{E}[X_m] = \mathbb{E}[X]
$$
$$
\mathbb{E}[X_m^2] = \mathbb{E}[X^2] + \frac{\mathbb{E}[X]}{m}
$$
$$
\mathbb{E}[X_m^n] - \mathbb{E}[X^n] = \frac{\mathbb{E}[X^{n-1}] d_{n, n-1}}{m} + o\left(\frac{1}{m}\right)
$$
 
$X_m$ converges to $X$ in Laplace-Stieltjes transformation hence also convergence in distribution as $m \to \infty$.
\end{theorem}
\begin{proof}
    See ref.\cite{he2022continuous}, Theorem 1.      
\end{proof}

\vspace{1mm}
Theorem~\ref{thm: dph_to_ph} provides a method to obtain the continuous extension of the cumulative reward together with the previous theorem. The next theorem shows this continuous extension can be approximated by an Erlang  mixture.
\begin{theorem}[Mixture of Erlang Distributions Approximation]
Let $F(x)$ be the cumulative distribution function (CDF) of a positive random variable. For a fixed $\beta>0$, define an approximate $\operatorname{CDF} F_\beta(x)$:
$$
F_\beta(x)=\sum_{j=1}^{\infty} p_j(\beta) \cdot F_{\operatorname{Erlang}(j, \beta)}(x)
, \quad x \geq 0
$$
\noindent
where 
 $p_j(\beta)=F(j \beta)-F((j-1) \beta), \quad j=1,2, \ldots$ 
and
$F_{\text {Erlang }(j, \beta)}(x)=1-\sum_{k=0}^{j-1} {e^{-x / \beta}(x / \beta)^k}/{k!}$
is the CDF of an Erlang distribution with shape parameter $j$ and scale parameter $\beta = 1/ \lambda$. Then, we have the convergence: $$\lim _{\beta \rightarrow 0} F_\beta(x)=F(x)$$
for any continuity point $x$ of $F$.
\end{theorem}

\begin{proof}
\noindent
We provide a simplified proof and a detailed proof see. ref~\cite{schweitzer1996stochastic}, Theorem 3.9.1.

\noindent For fixed $\beta$, let $U_{\beta, x}$ be a Poisson-distributed random variable with
$$
Pr\left\{U_{\beta, x}=k \beta\right\}=e^{-x / \beta} \frac{(x / \beta)^k}{k!}, \quad k=0,1, \ldots
$$
with expectation $\mathbb{E}\left[U_{\beta, x}\right]=x$ and variance $\mathbb{E}[U^2_{\beta, x}]=x \beta$. As $\beta \rightarrow 0, U_{\beta, x}$ concentrates around $x$ due to vanishing variance.
\noindent
For a bounded function $g(t)$, we show that:
$$
\lim _{\beta \rightarrow 0} \mathbb{E}\left[g\left(U_{\beta, x}\right)\right]=g(x)
$$
by bounding the error $\left|E\left[g\left(U_{\beta, x}\right)\right]-g(x)\right|$. Using Chebyshev's inequality, we control the tail probabilities, showing that as $\beta \rightarrow 0$, the contributions from $\left|U_{\beta, x}-x\right|$ become arbitrarily small.
\noindent
Applying this to $g(t)=F(t)$, the CDF of the random variable, we get:
$$
F(x)=\lim _{\beta \rightarrow 0} E\left[F\left(U_{\beta, x}\right)\right]
$$
\noindent which leads to:
$$
F(x)=\lim _{\beta \rightarrow 0} \sum_{j=1}^{\infty} p_j(\beta) \sum_{k=j}^{\infty} e^{-x / \beta} \frac{(x / \beta)^k}{k!}
$$
yielding the desired convergence:
$$
\lim _{\beta \rightarrow 0} F_\beta(x)=F(x)
$$
\end{proof}

\noindent
The combination of discretizing rewards, approximating DPH distributions with continuous phase-type distributions, and representing these as Erlang mixtures provides a solid foundation for approximating the cumulative reward distribution of a DTMC. In real applications, selecting the appropriate number of Erlang distributions in the mixture to get an acceptable approximation is an empirical task \cite{Verbelen_Gong_Antonio_Badescu_Lin_2015}.

 \section{Evaluation}\label{SecEval}
We evaluate our moment-matching approximation method with error analysis using case studies adapted from PRISM~\cite{kwiatkowska2002prism} and used also in previous literature~\cite{elsayed2024distributional}. Our implementation employs an Erlang mixture model with moment matching and maximum entropy using SciPy SLSQP~\cite{virtanen2020scipy}, and Gurobi~\cite{gurobi} for optimization, while the linear systems of equations required to compute the moments of the reward process are solved with Numpy~\cite{harris2020array}. The primary goal is to empirically assess the approximation performance of our method against the empirical distribution and traditional histogram-based approximation from the perspective of accuracy and computational tractability, as well as discuss hyper-parameters and overheads.

\vspace{1mm}
\noindent\textbf{Benchmarks.} 
The benchmarks, adapted from PRISM case studies~\cite{kwiatkowska2002prism}, include scenarios like message synchronization and Herman protocol~\cite{herman1990probabilistic}, which are modeled as DTMCs; for the MDP benchmarks (such as betting games, and grid world navigation), we evaluate the DTMC induced by the optimal policy produced by PRISM (reward schemes defined for the original models). We evaluate our method also on two models with continuous rewards, which cannot be analyzed by previous work~\cite{elsayed2024distributional}: UAV (our example model) and an adaptation of the Future Investor model~\cite{mciver2007results}.
The benchmark subjects cover a diverse range of state space sizes and transition structures, as summarized in Table~\ref{tab:benchmarks}.

\bigskip
\begin{table}[!ht]
    \centering
    \begin{tabular}{@{}lllllll@{}}
    \toprule
    \textbf{Subject}& \textbf{States} & \textbf{Transitions} & $d_{\texttt{KS}}$ (\cite{elsayed2024distributional})& $d_{\texttt{KS}}$ (our) \\
    \midrule

    GridWorld Navigation & 396 & 6,291 & 0.42 & 0.14\\
    Betting Game & 891   & 4,216 & 0.16 & 0.09\\
    DeepSea Treasure & 1,756 & 5, 672  & 0.28 & 0.13\\
    Herman & 2,048  & 177,148 & 0.27 & 0.09\\
    LeaderSync & 1,050   & 1,292 & 0.01 & 0.07 \\

    \midrule[\heavyrulewidth]

    UAV Energy & 13 & 29 & - & 0.05 \\
    Future Investor &  2,899 & 7,967 &- &0.17 \\
    \bottomrule
    \end{tabular}
    \vspace{1mm}
    \caption{Benchmark subjects with discrete rewards (top) and continuous rewards (bottom).}
    \label{tab:benchmarks}
\end{table}

\noindent \textbf{Experimental Settings.} 
We construct the approximate reward distribution from the initial state of each model using SLSQP~\cite{virtanen2020scipy}. To evaluate the accuracy of the approximation, we construct a baseline empirical reward distribution by running one million simulations of the models and collecting the reward accumulated for each run. For models using integer rewards, we also compare with the histogram distributions constructed by~\cite{elsayed2024distributional}.

To align the approximation problem with the truncated Stieltjes moment problem on \( (0, +\infty) \), we shift the location of the Erlang distributions to \( loc = \mu -\sigma \) ($\sigma$ being the standard deviation of the reward distribution), following the location setting in~\cite{gavriliadis2012truncated}. Based on~\cite{johnson1989matching}, the minimum number of Erlang components in the mixture required to match $K$ moments is $\lfloor K/2 \rfloor + 1$. To allow more flexibility in the optimization, unless otherwise specified, we use \( k = 3 \) moments and \( n = 3 \) components. 
\noindent The common rate parameter $\lambda$ is limited for optimization within the range $(0.01, 50)$ (based on our experience, the range was large enough to include the optimization results). The weights are initialized uniformly: $\omega_{i} = 1 / n$. The parameter $\gamma$ balancing the trade-off between moment matching and maximum entropy is set to $1$ for all the experiments.
As discussed in Section~\ref{secDistributionApproximation}, we restrict the Erlang shapes into a power-based sequence of shape parameters \( a_i = 3^i \), \( i \leq n \), which ensures a wide spread of options in the weighted selection of shapes~\cite{verbelen2013phase} to cover head, multi-modal, and tail behaviors-- we will evaluate the robustness of this choice of heuristics and compare with alternative choices in Section~\ref{secHyperparameters}.

\subsection{Approximation Accuracy}\label{secApproximationAccuracy}
\noindent To evaluate the accuracy of approximate distributions against the  baseline empirical distributions, we use the Kolmogorov-Smirnov (KS) statistic metric. KS measures the largest difference between the empirical CDF of the true values and the approximated CDFs with $D_{\text{KS}} = \sup_x \left| F_{\text{empirical}}(x) - F_{\text{approx}}(x) \right|$. The KS statistic is chosen for its focus on capturing the maximum discrepancy between distributions, making it well-suited and widely used for comparing distributions and evaluating approximation quality~\cite{massey1951kolmogorov, hong1991kolmogorov}.

\begin{figure}[ht]
\centering
\subfigure[LeaderSync (PDF)]{\includegraphics[width=0.24\textwidth]{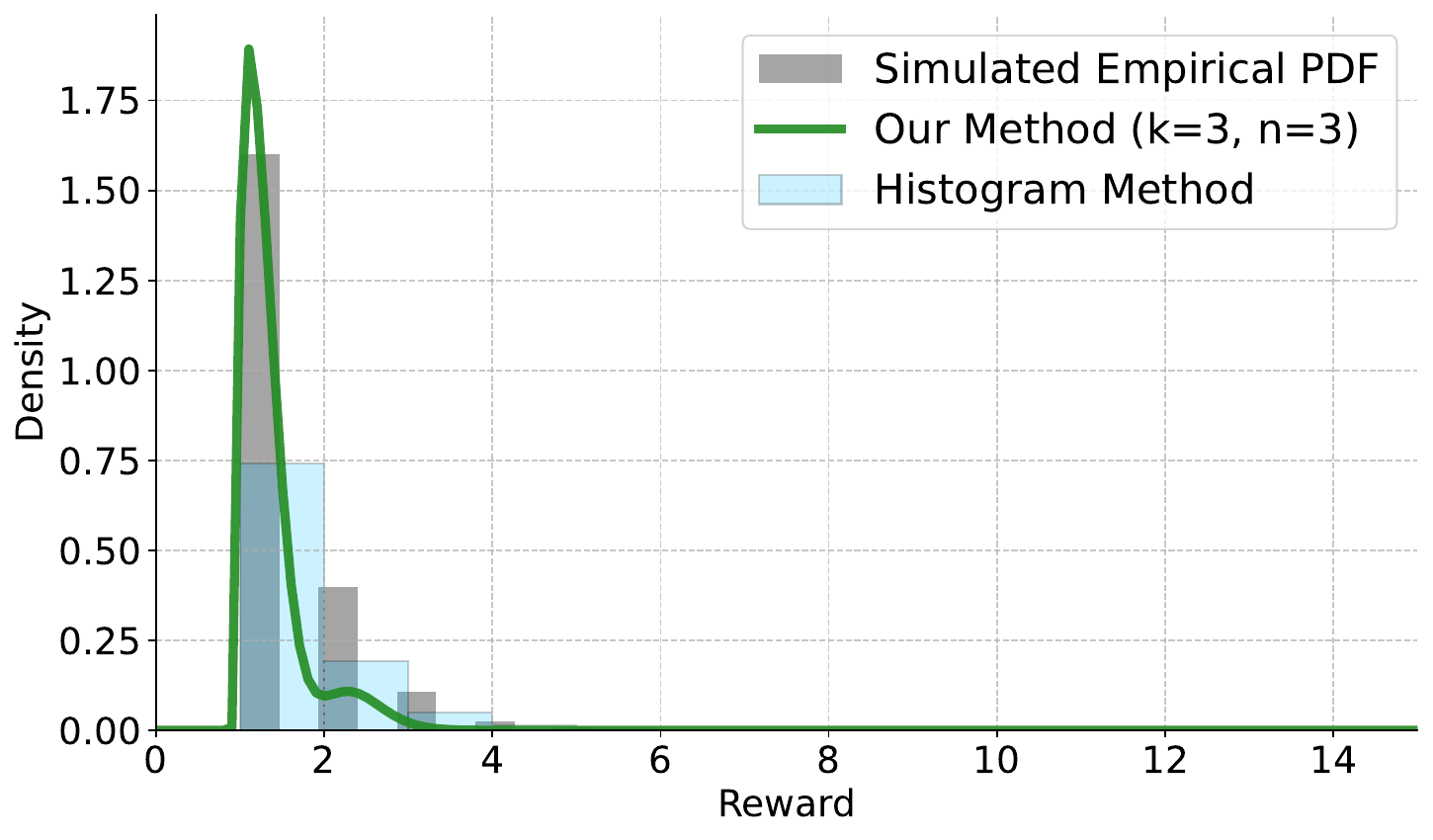}}\subfigure[Herman (PDF)]{\includegraphics[width=0.24\textwidth]{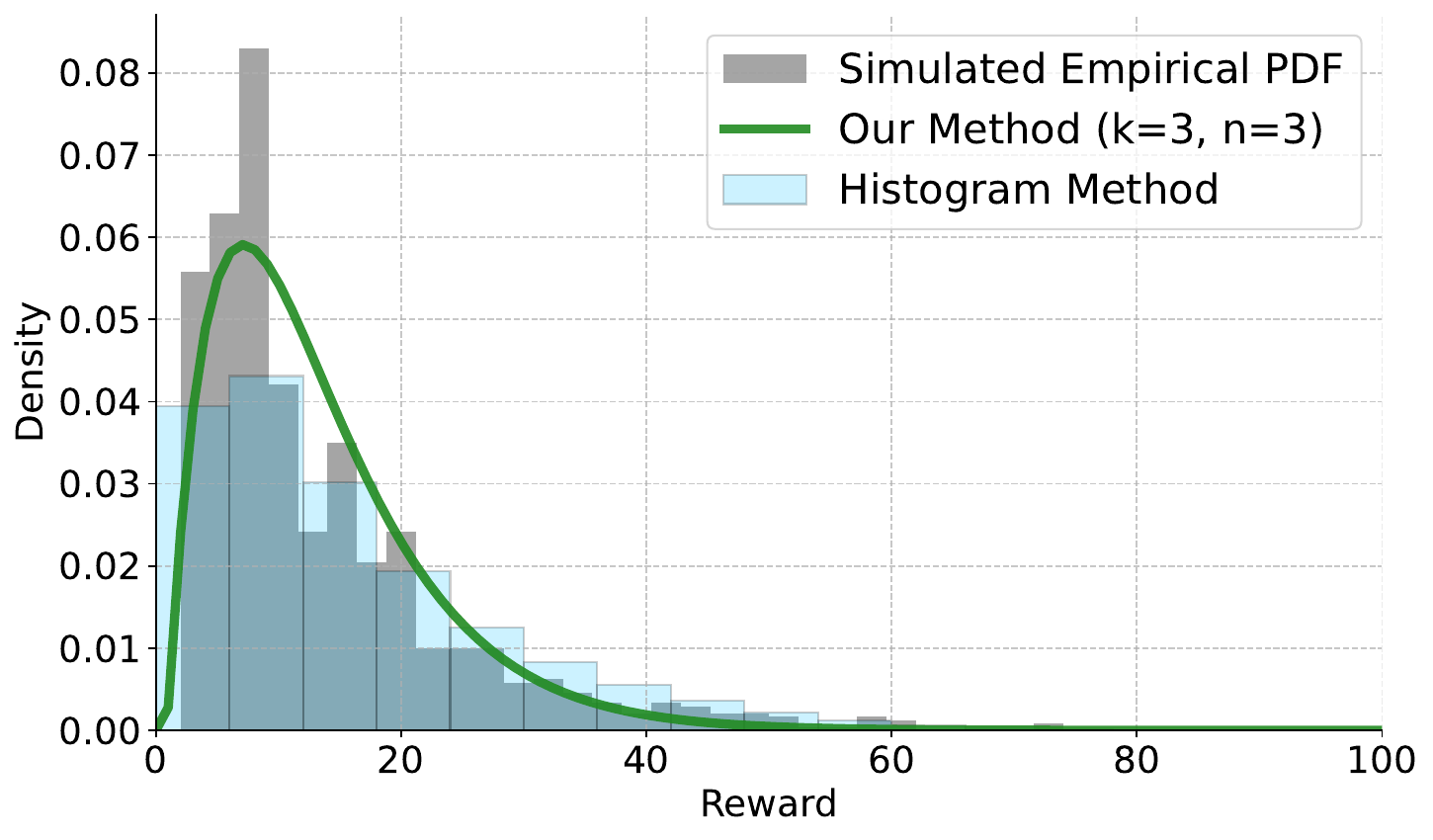}}\\
\subfigure[LeaderSync (CDF)]{\includegraphics[width=0.24\textwidth]{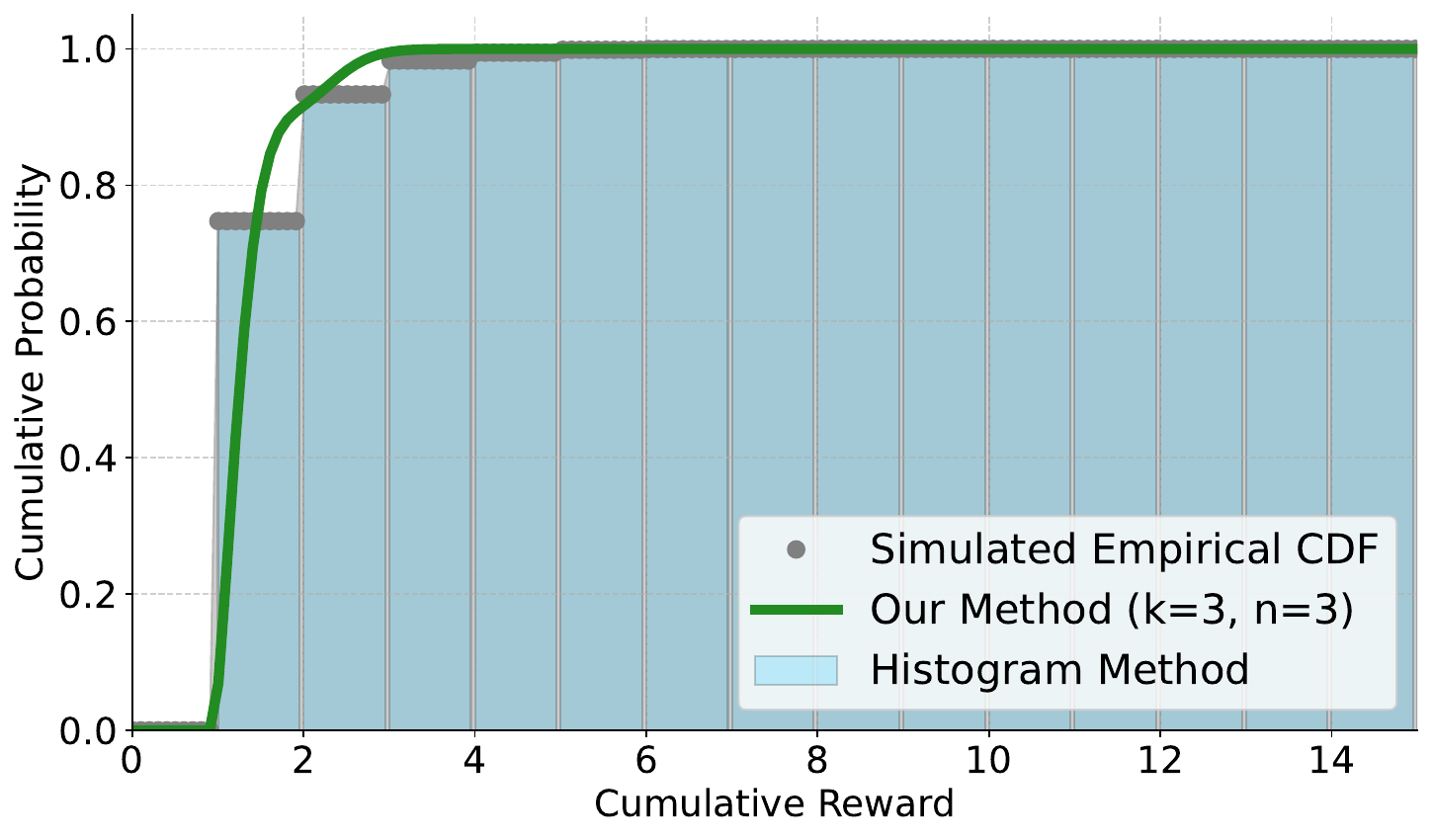}}\subfigure[Herman (CDF)]{\includegraphics[width=0.24\textwidth]{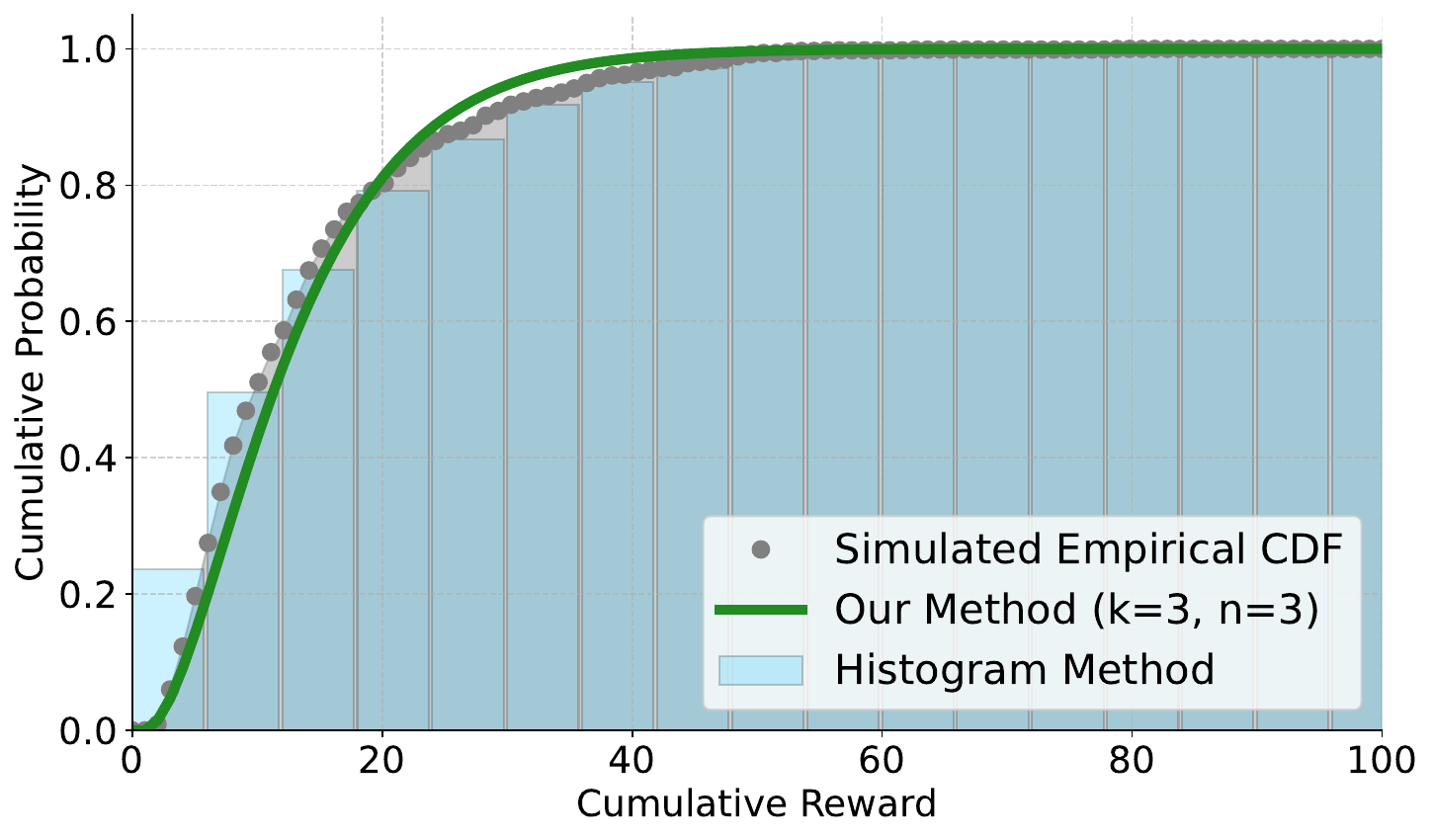}}
\caption{PDF and CDF for discrete reward space subjects (figure 1 of 2)}
\label{fig:pdfcdfdiscrete1}
\end{figure}

\begin{figure*}[tb]
\centering
\subfigure[GridWorld Navigation (PDF)]{\includegraphics[width=0.32\textwidth]{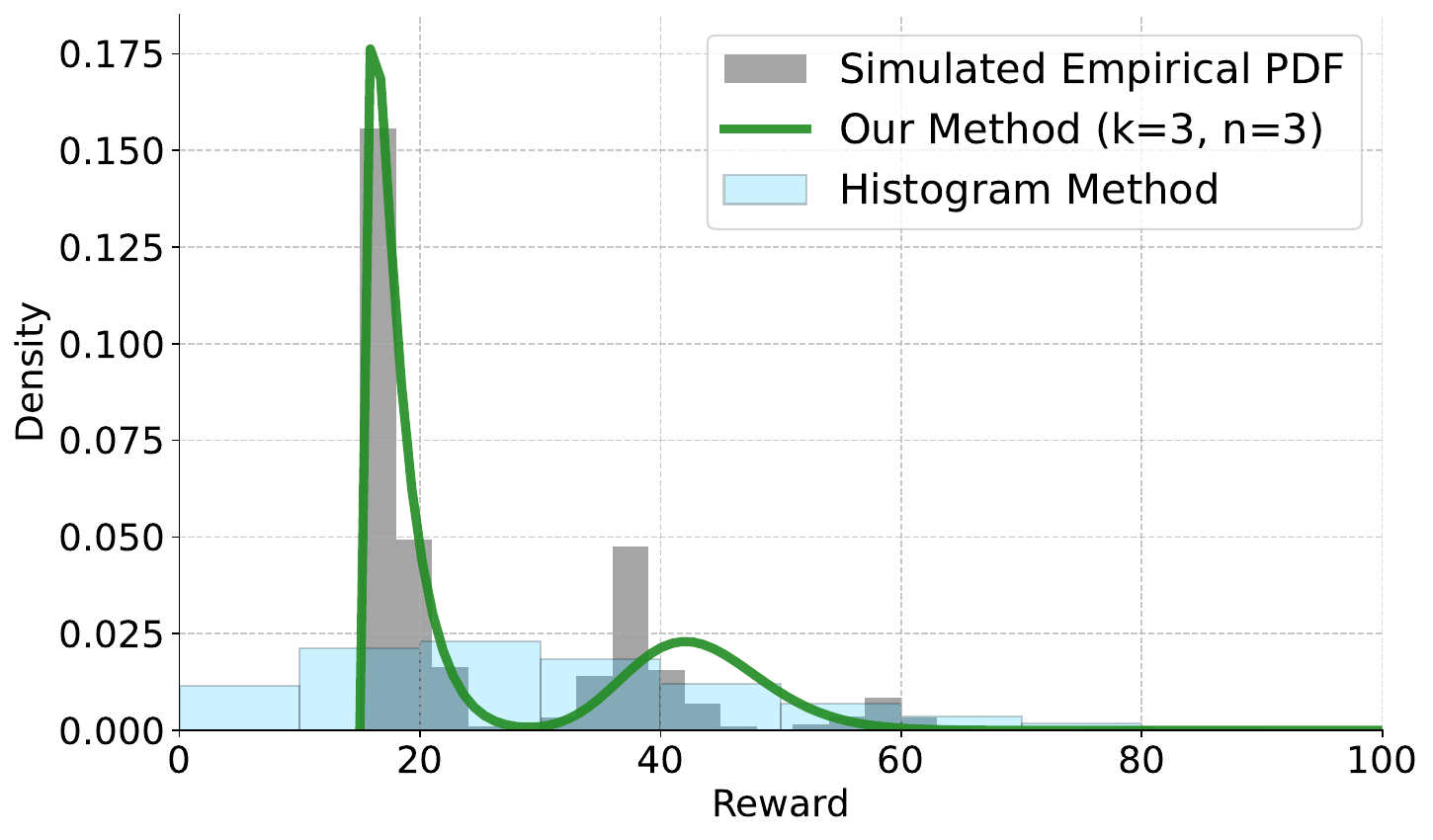}}
\subfigure[Betting Games (PDF)]{\includegraphics[width=0.32\textwidth]{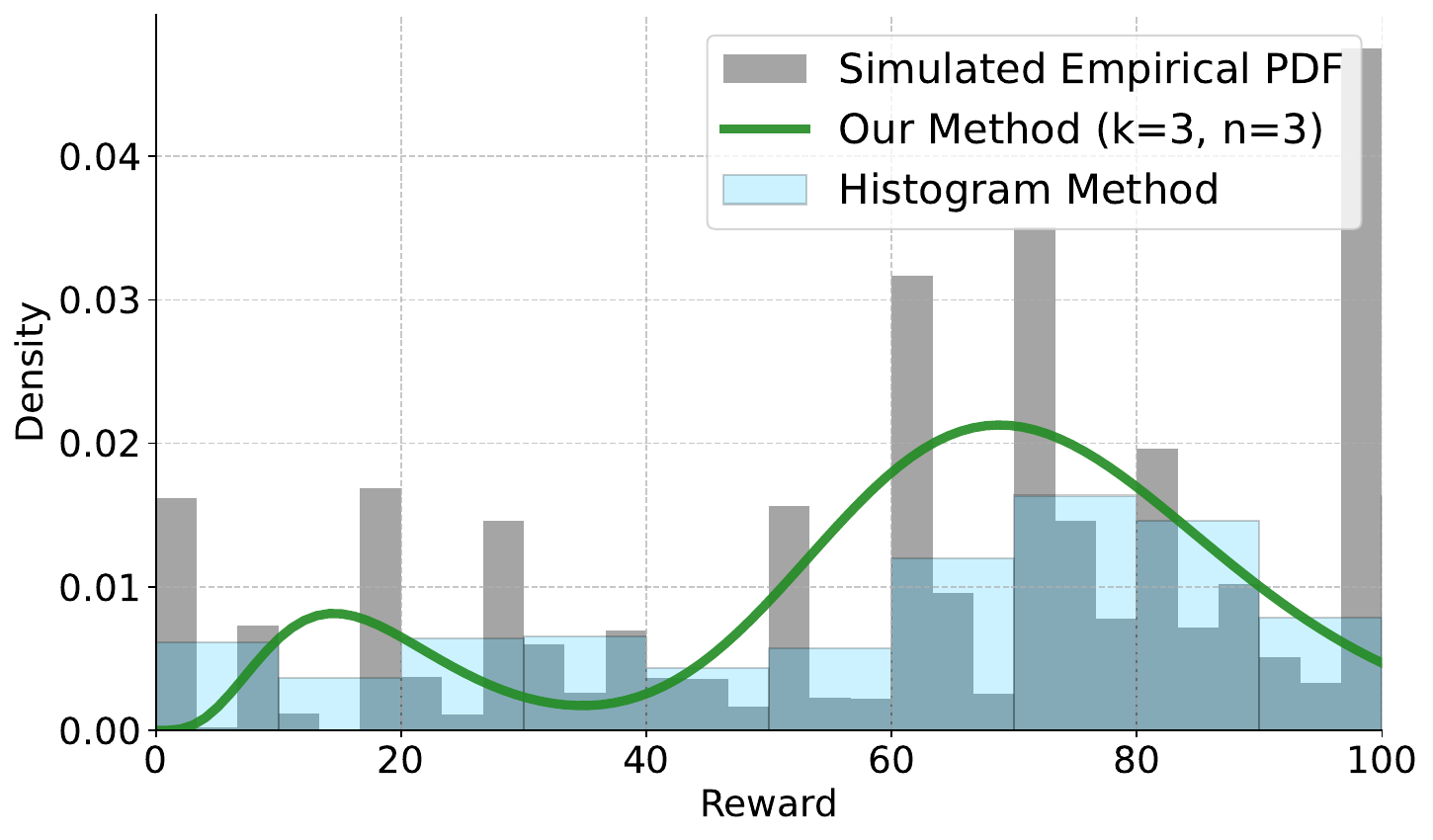}}
\subfigure[DeepSea Treasure (PDF)]{\includegraphics[width=0.32\textwidth]{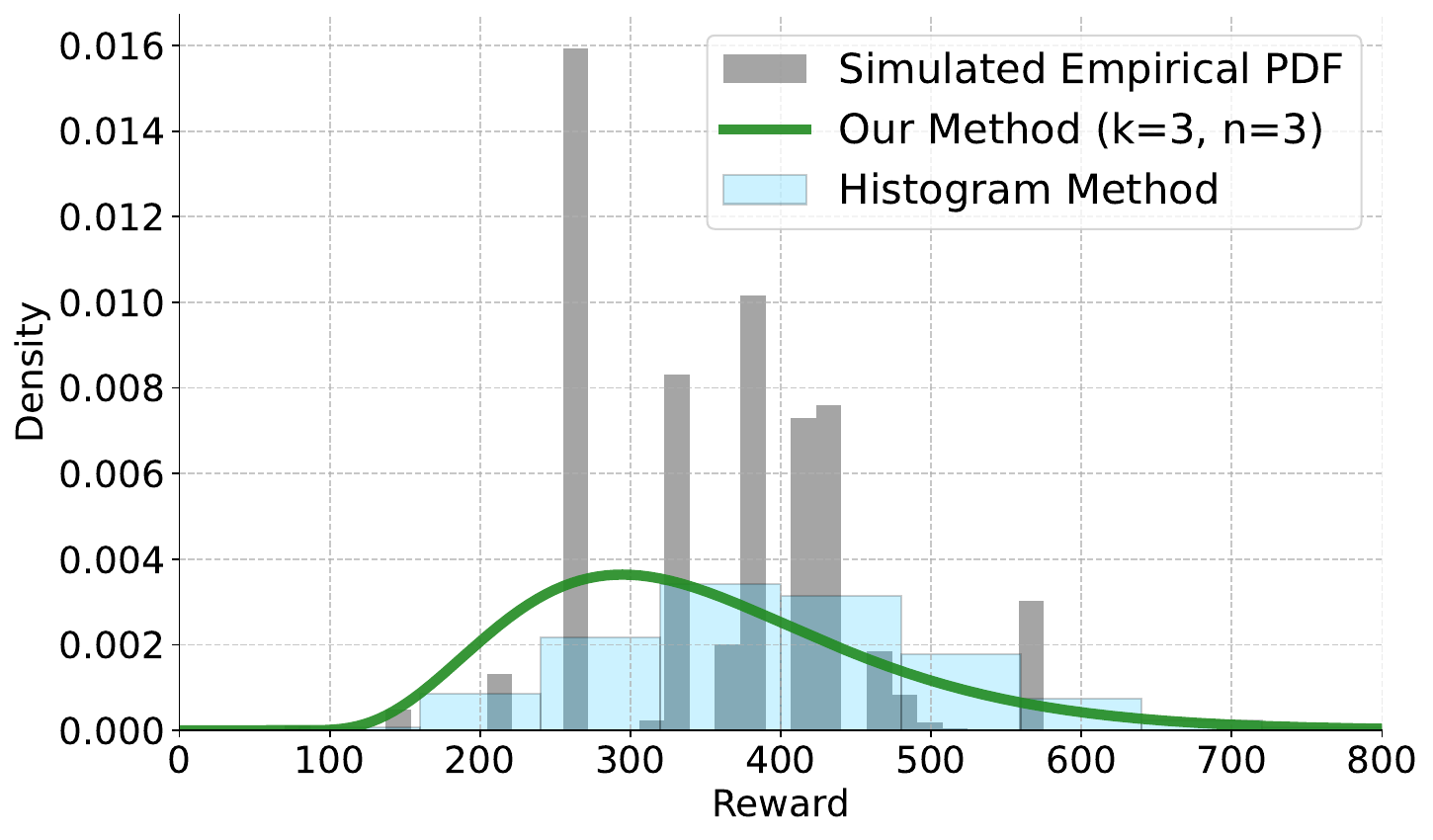}}\\
\subfigure[GridWorld Navigation (CDF)]{\includegraphics[width=0.32\textwidth]{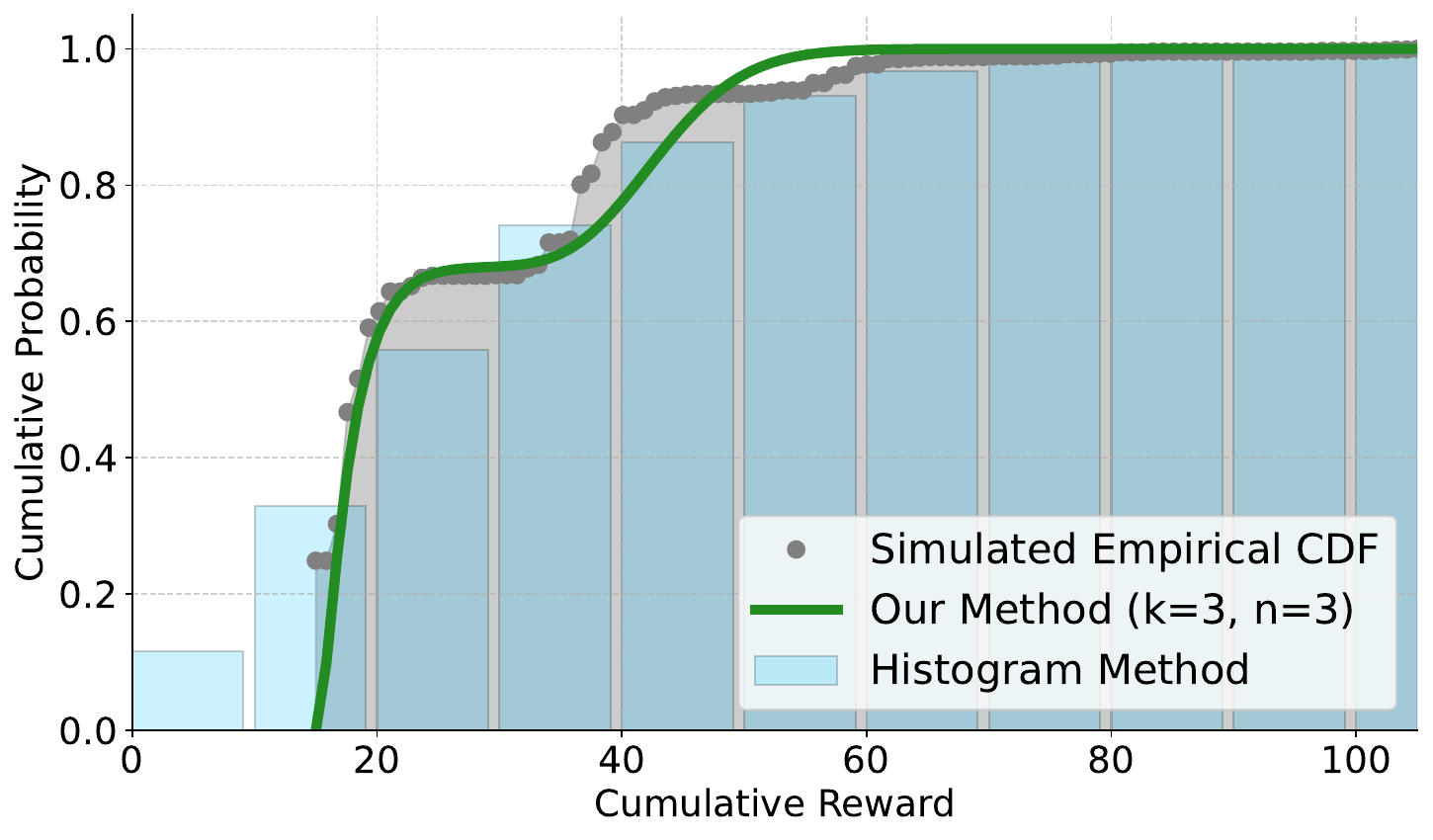}}
\subfigure[Betting Games (CDF)]{\includegraphics[width=0.32\textwidth]{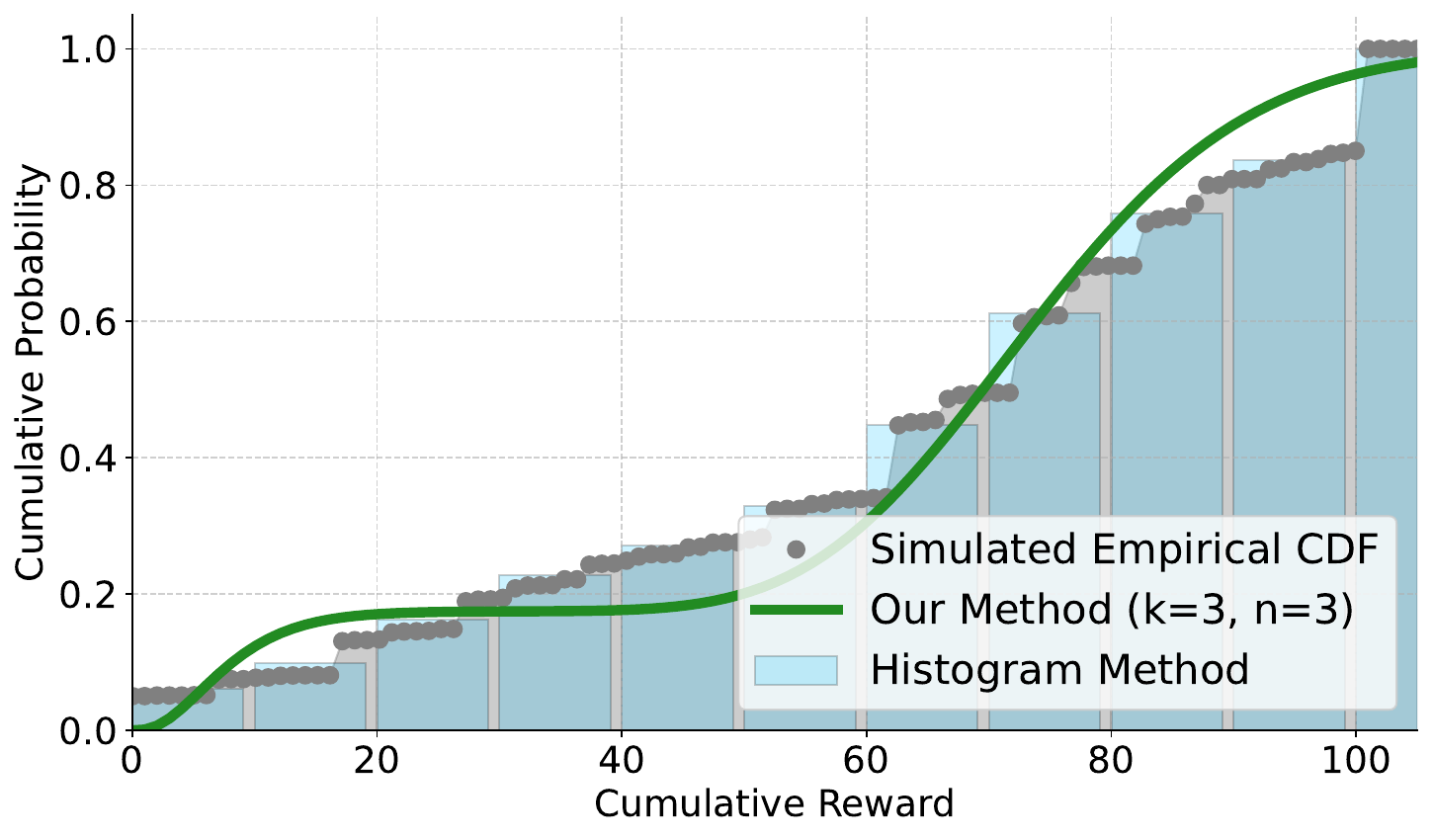}}
\subfigure[DeepSea Treasure (CDF)]{\includegraphics[width=0.32\textwidth]{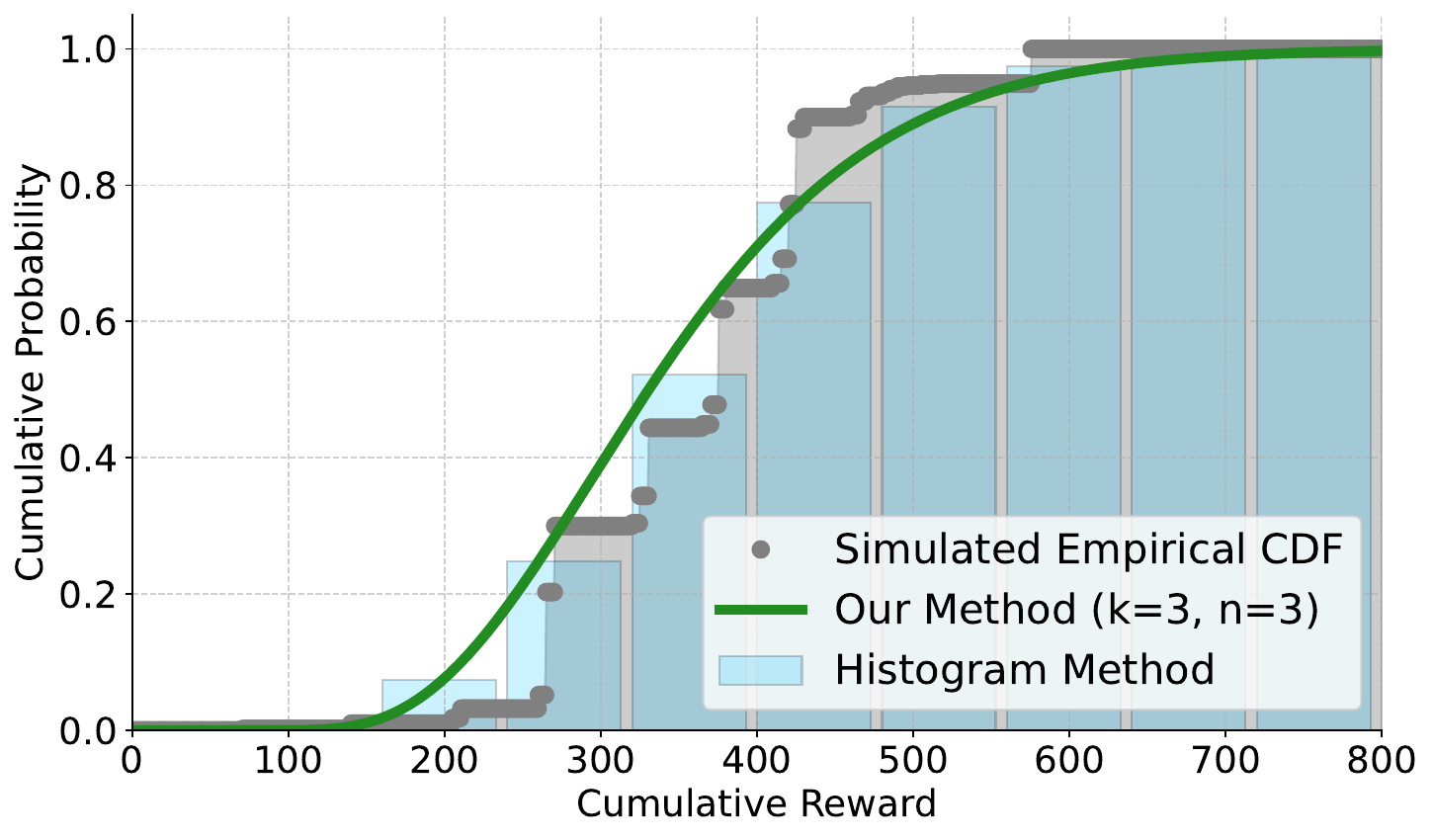}}
\caption{PDF and CDF for discrete reward space subjects (figure 2 of 2)}
\label{fig:pdfcdfdiscrete2}
\end{figure*}

\begin{figure}[ht]
\centering

\subfigure[UAV Energy (PDF)]{\includegraphics[width=0.24\textwidth]{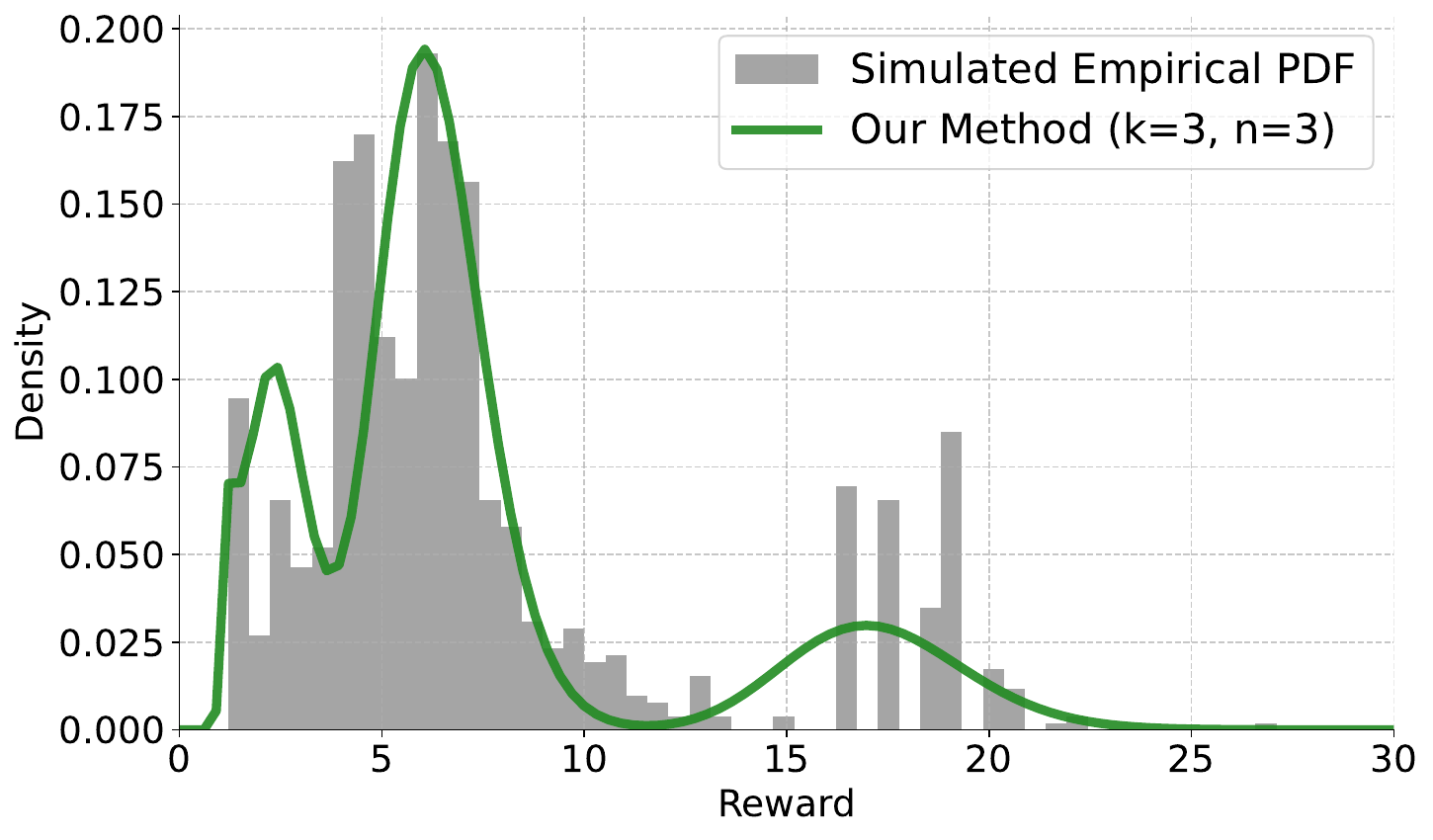}}
\subfigure[Future Investor (PDF)]{\includegraphics[width=0.24\textwidth]{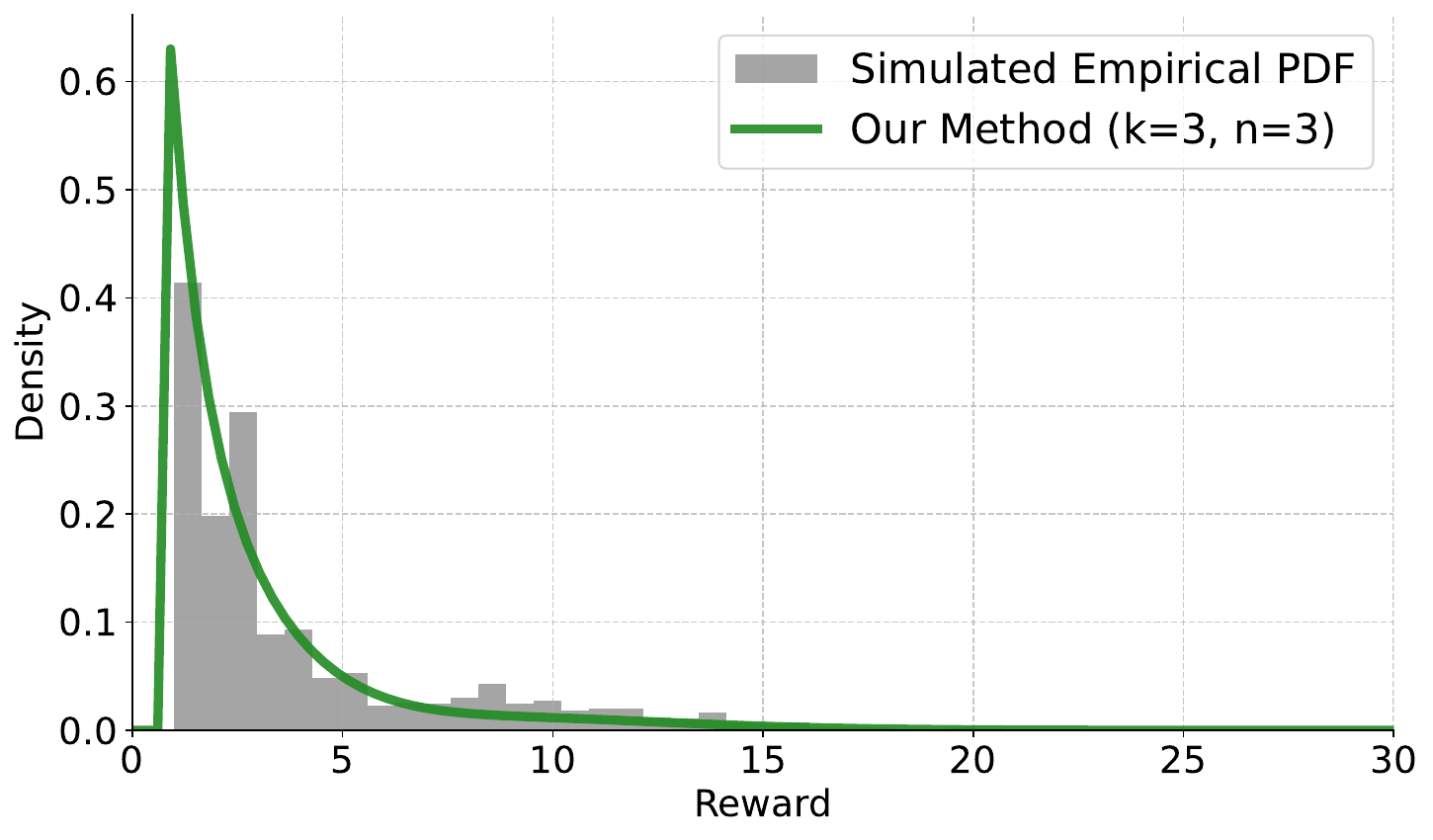}}\\
\subfigure[UAV Energy (CDF)]{\includegraphics[width=0.24\textwidth]{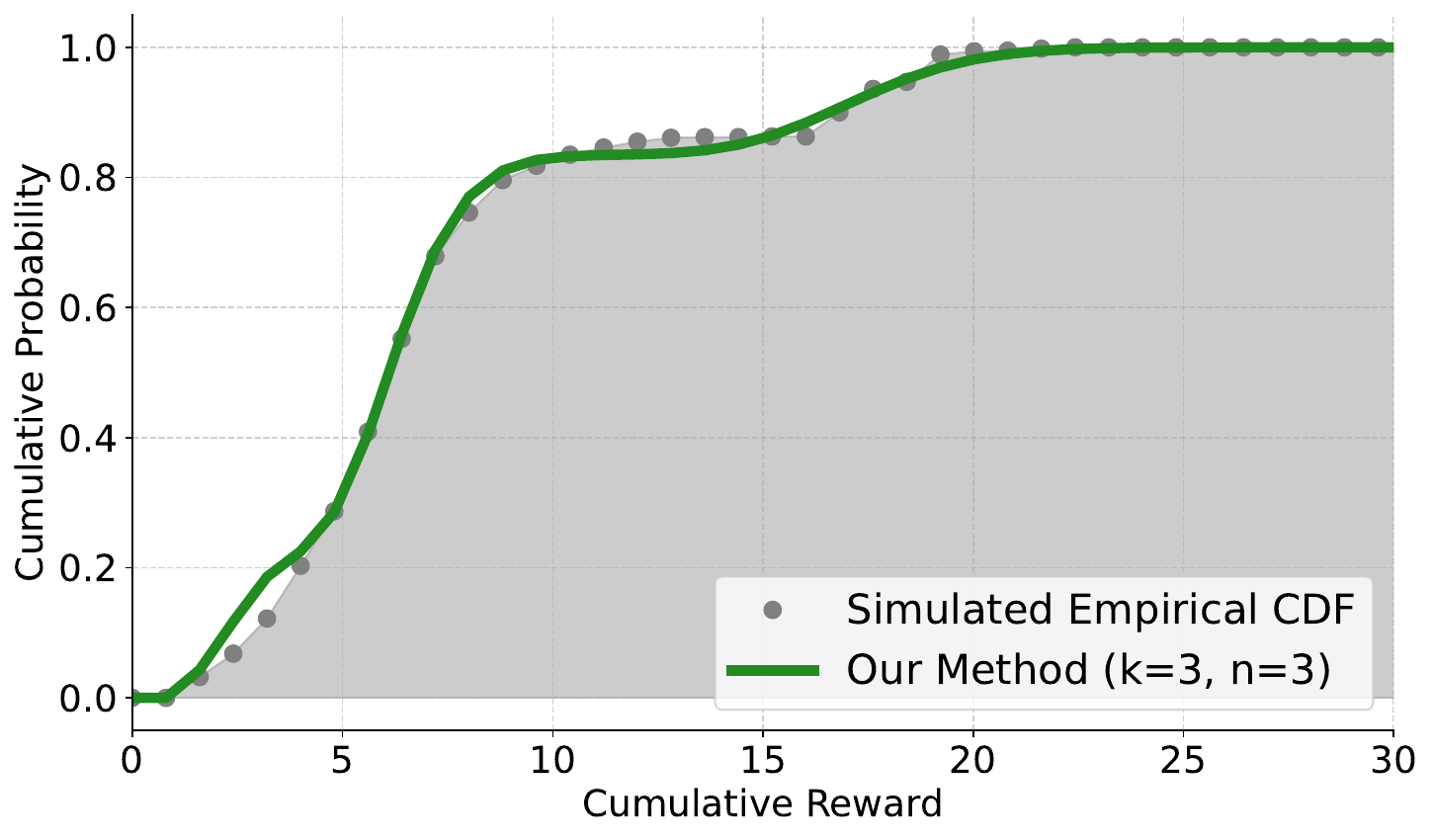}}
\subfigure[Future Investor (CDF)]{\includegraphics[width=0.24\textwidth]{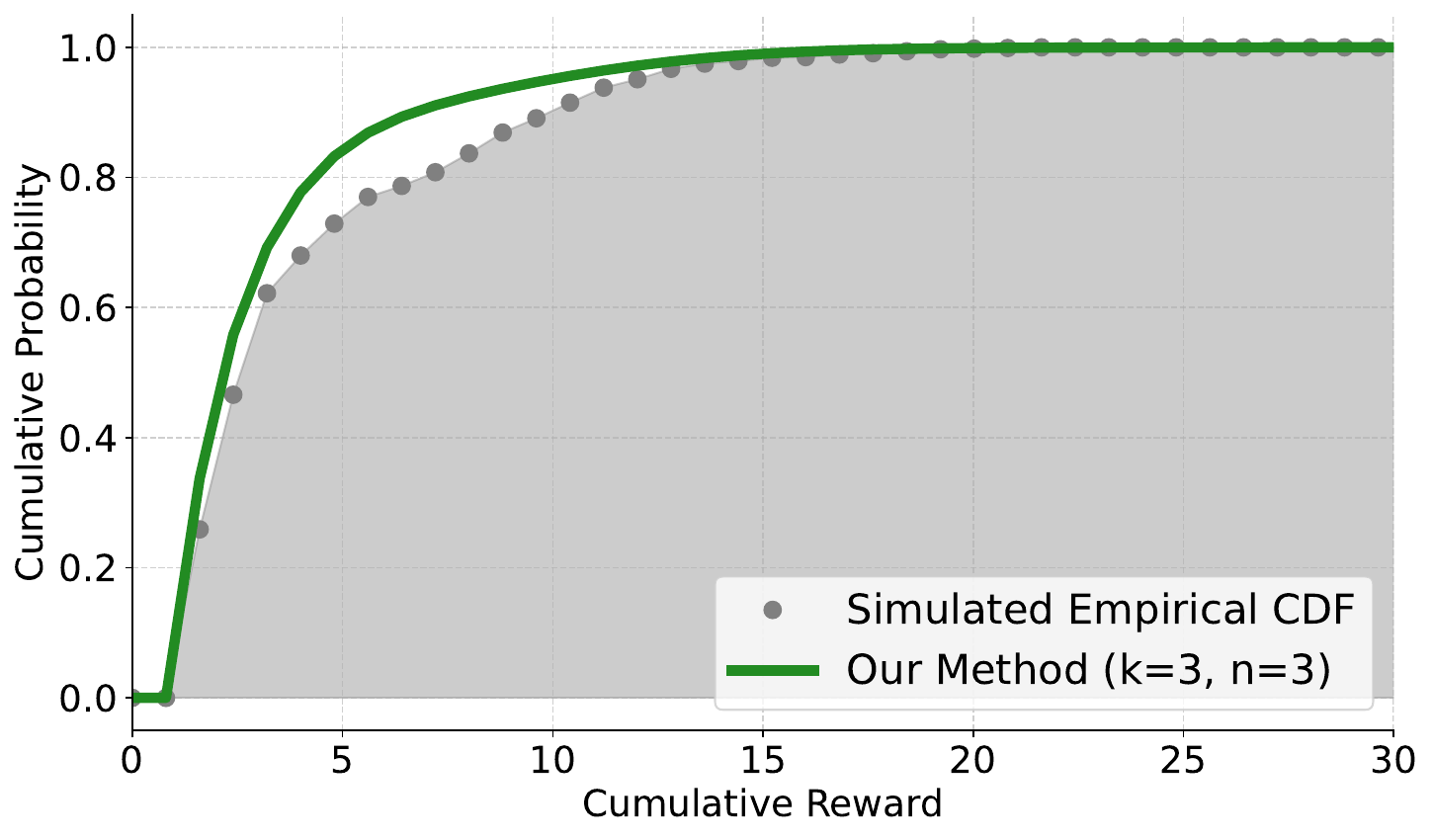}}

\caption{PDF and CDF for continuous reward space subjects}
\label{fig:continuous}
\end{figure}

\noindent\textbf{Discrete Reward Space.}
Figures~\ref{fig:pdfcdfdiscrete1} and~\ref{fig:pdfcdfdiscrete2} present an evaluation of our method on the discrete reward space subjects. The blue histograms represent the discrete approximation constructed by~\cite{elsayed2024distributional}, 
referred to as \emph{histogram method} in the figures. The gray histogram represents the empirical baseline, constructed from one million sample runs of the model.
The KS comparison measures are reported in Table~\ref{tab:benchmarks}.

Overall, the two methods show comparable performance, with our method achieving lower KS divergence for all subjects except \texttt{LeadSync}. In the \texttt{LeadSync} case, support of the probability distribution is in $(1, 10] \cap \mathbb{Z}$, allowing the histogram to fully capture the true distribution without loss of information. This result also shows that the accuracy of histogram-based methods is highly dependent on bin width, with finer binning improving precision but increasing the computational cost. 
In the \texttt{Herman} scenario, our method shows a closer alignment with the empirical distribution, as also captured by the lower KS metric (Table~\ref{tab:benchmarks}).
Figure~\ref{fig:pdfcdfdiscrete2} includes the approximations for the remaining discrete reward space subjects, where the probability density function (PDF) includes multiple modalities, and the achievable reward values are generally more scattered. In these cases, the Erlang-based approximation offers a smoother fit compared to histograms and outperforms~\cite{elsayed2024distributional}, especially in sparsely populated regions like \texttt{GridWorld Navigation}, as confirmed by the $d_{\texttt{KS}}$ metric. Although the performance of any histogram-based method can be improved with finer bins~\cite{bellemare2017distributional}, scalability and sparse-dense ranges issues remain problematic in continuous and larger finite state spaces.

\vspace{1mm}
\noindent\textbf{Continuous Reward Space.} Figure~\ref{fig:continuous} shows the evaluation of our method on the two scenarios with continuous reward space: \texttt{UAV Energy} and \texttt{Future Investor}. The empirical distribution is represented by the gray histogram. For easier comparison across the subjects, we used $K=3$ moments and $n=3$ components of the Erlang mixture. In the case of \texttt{Future Investor}, as noticeable by visual inspection, and reflected in a KS metric of 0.17, this $(K,n)$ settings tend to overapproximate the CDF around the head of the distribution, catching up after $x\geq 12$. In the next section, we will dive deeper into how such approximation can be improved by increasing $K$ and $n$.

\begin{figure*}[h]
\centering
\subfigure[PDF with $k=3$ and $n = 3, \dots, 9$]{\includegraphics[width=0.32\textwidth]{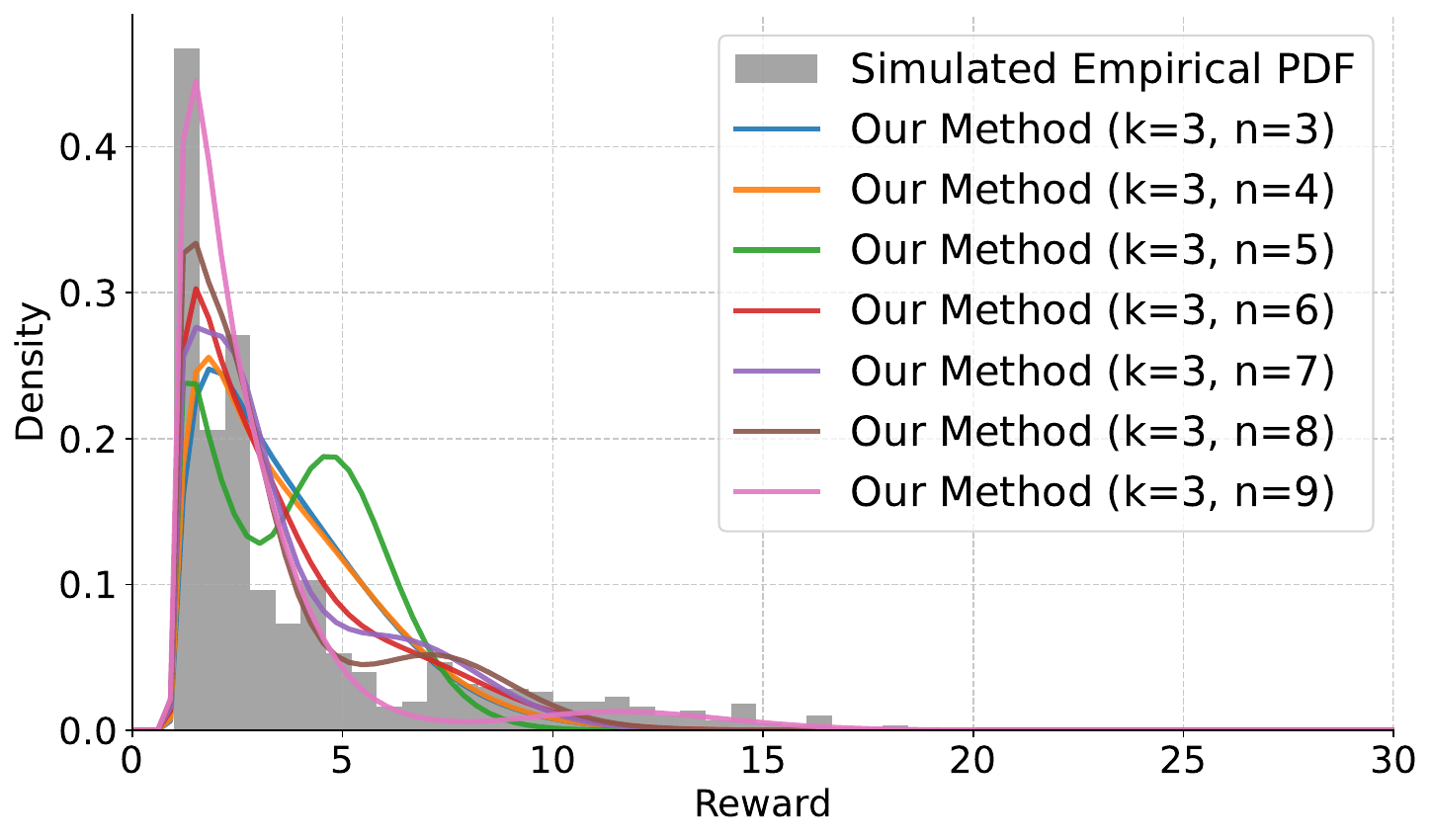}}
\subfigure[PDF with $k=4$ and $n = 3, \dots, 9$]{\includegraphics[width=0.32\textwidth]{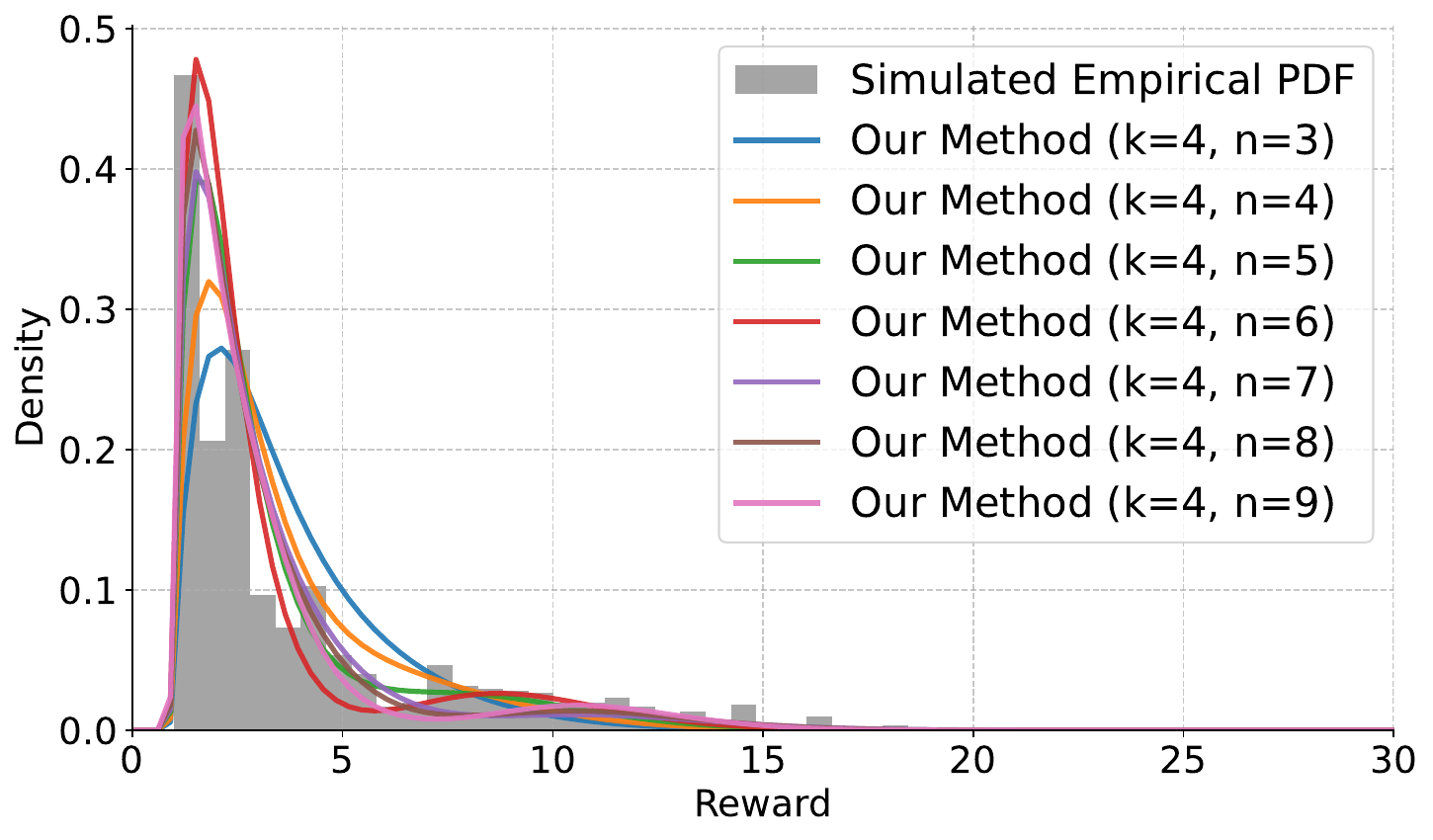}}
\subfigure[PDF with $k=5$ and $n = 3, \dots, 9$]{\includegraphics[width=0.32\textwidth]{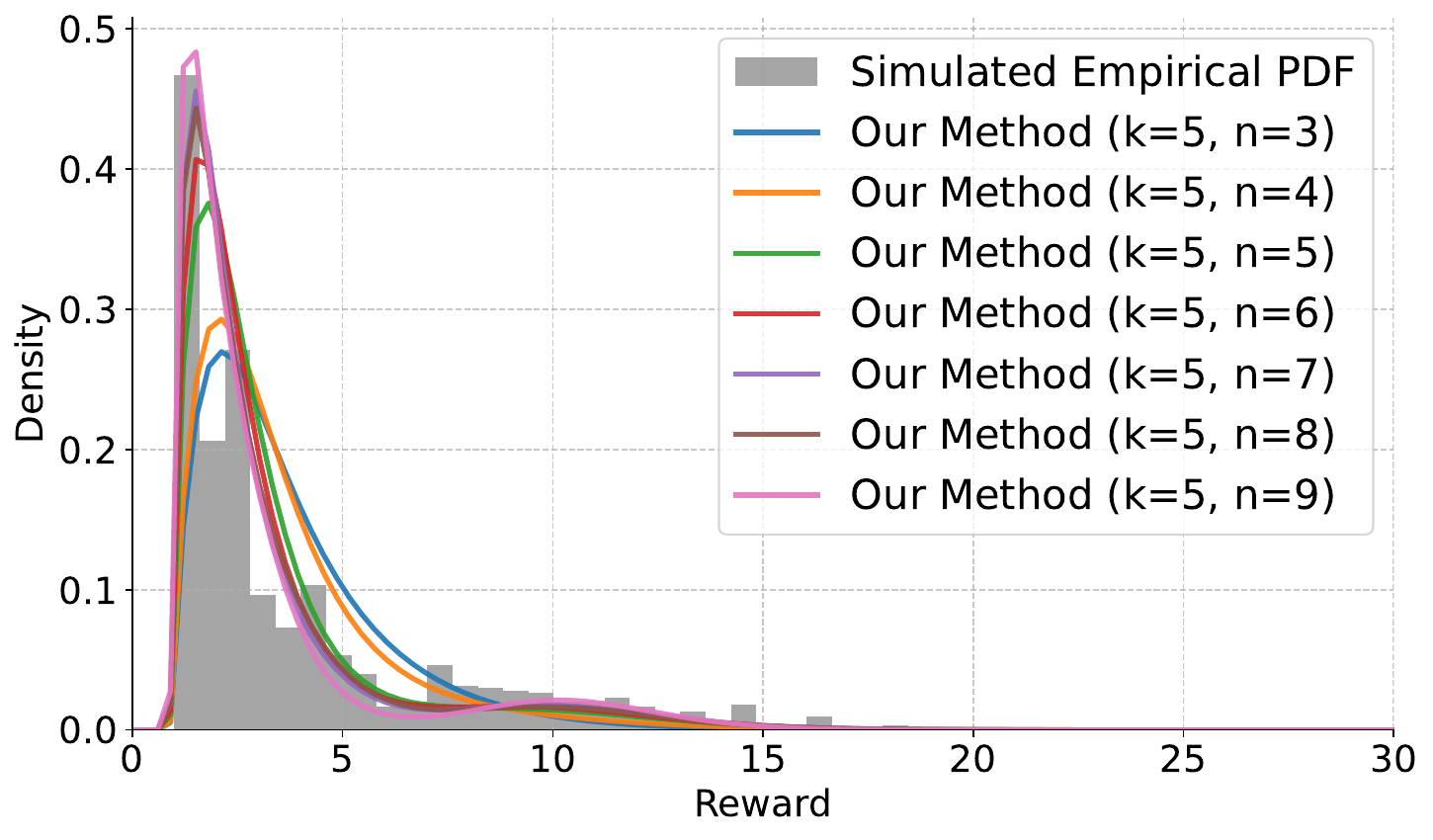}}\\
\subfigure[CDF with $k=3$ and $n = 3, \dots, 9$]{\includegraphics[width=0.32\textwidth]{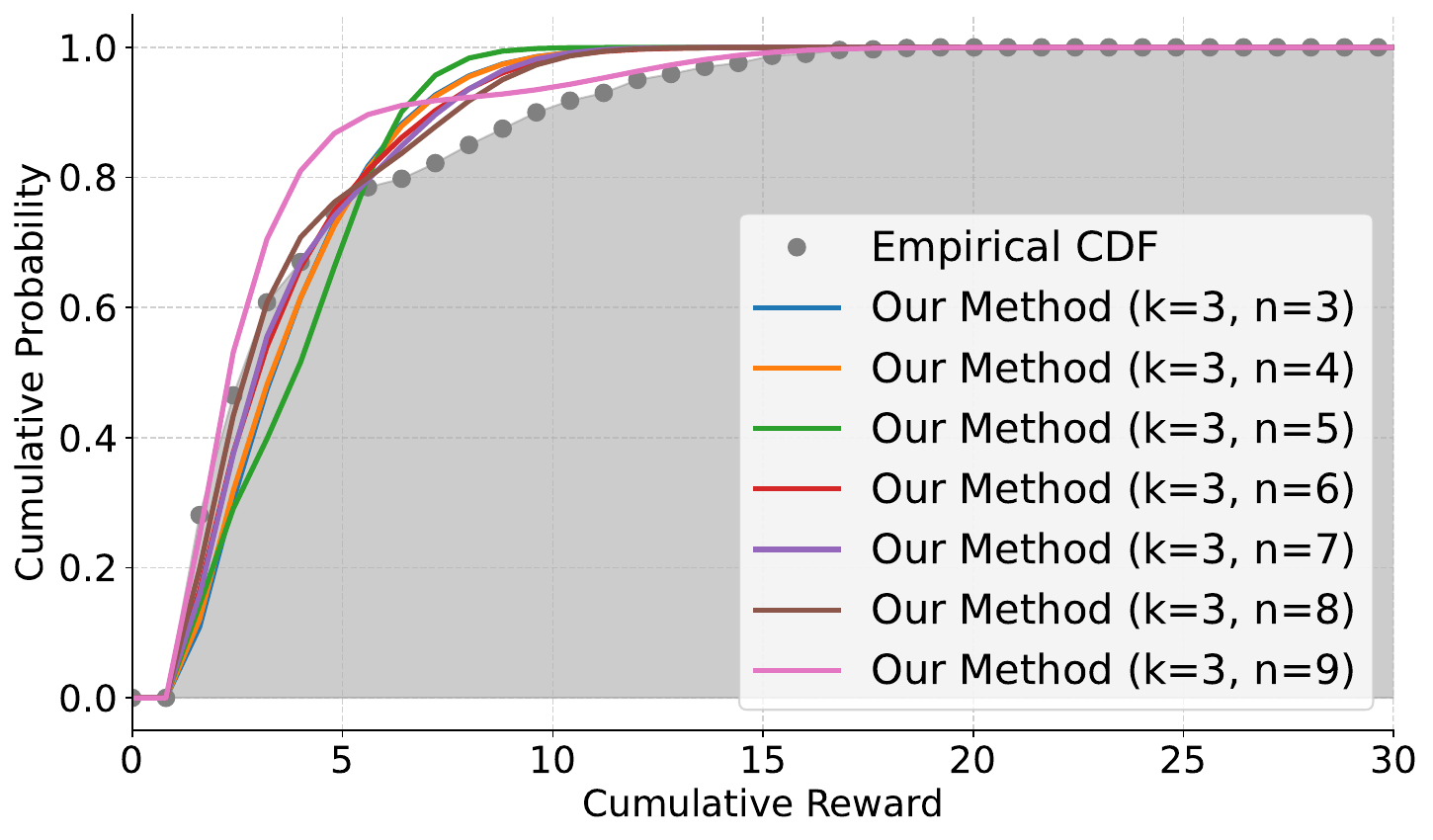}}
\subfigure[CDF with $k=4$ and $n = 3, \dots, 9$]{\includegraphics[width=0.32\textwidth]{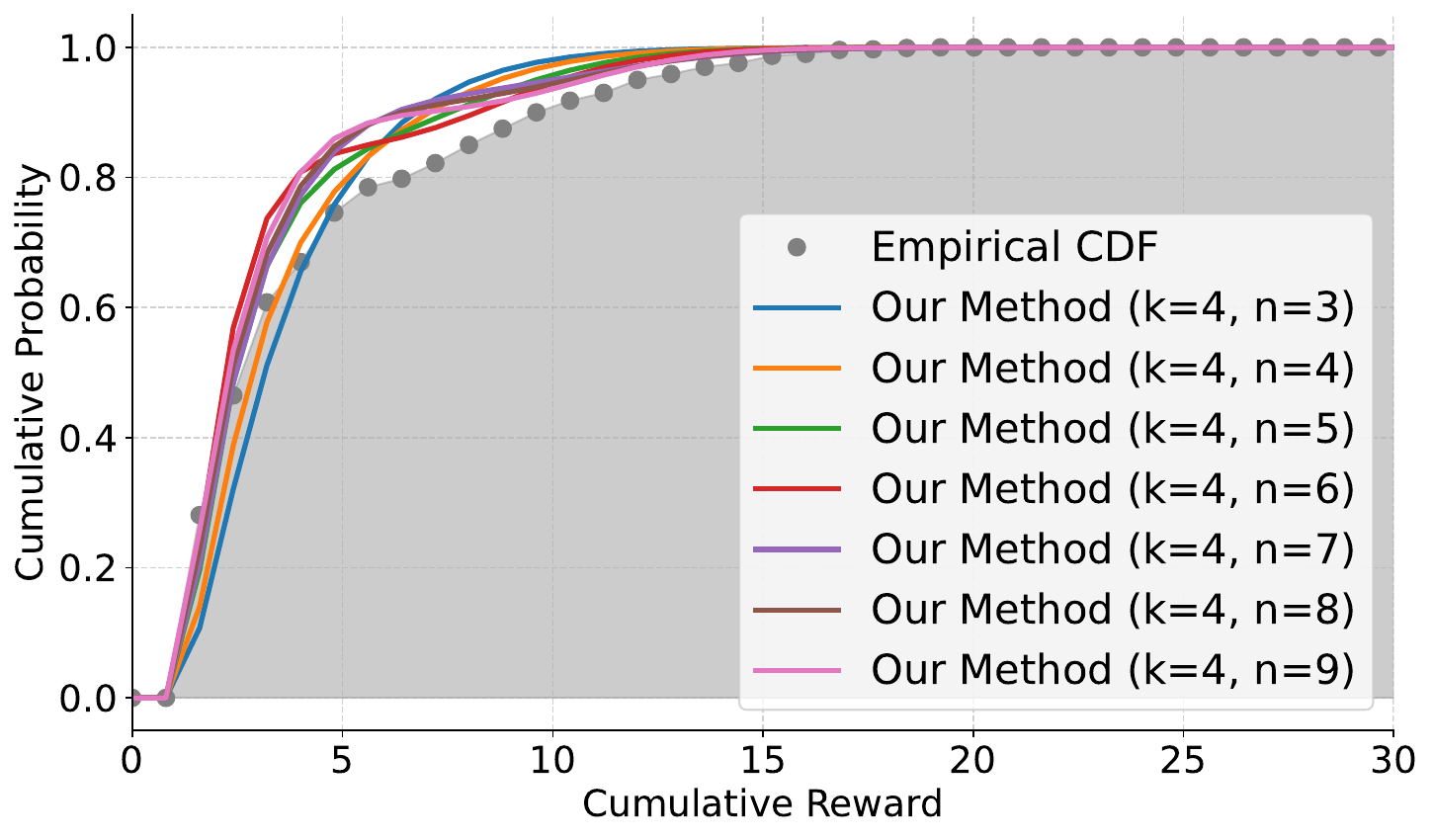}}
\subfigure[CDF with $k=5$ and $n = 3, \dots, 9$]{\includegraphics[width=0.32\textwidth]{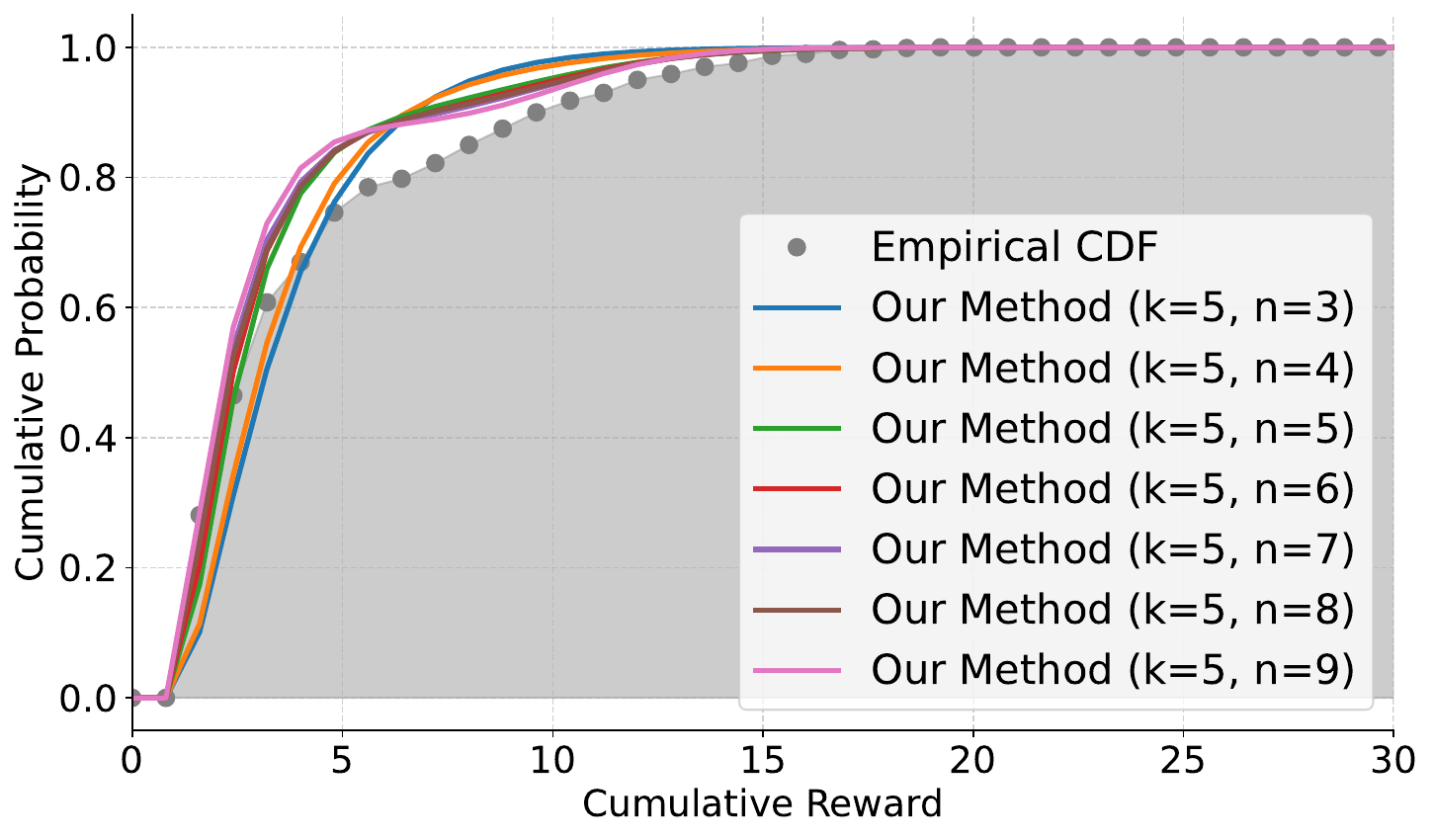}}
\caption{Approximated Probabilities Distribution of Financial Market with different number of mixtures $n$}
\label{fig:hyper}
\end{figure*}

\bigskip
\begin{table}[h]
    \centering
    \resizebox{0.43\textwidth}{!}{\begin{tabular}{lccccc}
    \toprule
    \textbf{$K$} & \textbf{$n$} & \textbf{$T_{\texttt{opt}}$ (s)} & \textbf{$T_{\texttt{total}}$ (s)} & \textbf{Iterations} & \textbf{$D_{\texttt{KS}}$} \\
    \midrule
    \multirow{7}{*}{3} & 3 & 0.17 & \textbf{1.42} & 36  & 0.17  \\
    & 4 & 0.19 & 1.44 & 43  & 0.17 \\
    & 5 & 0.31 & 1.56 & 57  & 0.26 \\
    & 6 & 0.23 & 1.48 & 57  & 0.16 \\
    & 7 & 0.44 & 1.69 & 63  & 0.15 \\
    & 8 & 0.62 & 1.87 & 80  & \textbf{0.14} \\
    & 9 & 1.65 & 2.91 & 92  & 0.21 \\
    \midrule
    \multirow{7}{*}{4} & 3 & 0.59 & 2.20 & 69  & 0.12 \\
    & 4 & 0.73 & 2.34 & 72  & 0.10 \\
    & 5 & 0.87 & 2.48 & 74  & 0.10 \\
    & 6 & 0.82 & 2.43 & 75  & \textbf{0.07} \\
    & 7 & 0.39 & \textbf{2.00} & 71  & 0.09 \\
    & 8 & 0.92 & 2.53 & 79  & 0.09 \\
    & 9 & 0.47 & 2.08 & 88  & 0.09 \\
    \midrule
    \multirow{7}{*}{5} & 3 & 2.14 & 4.24 & 92  & 0.17 \\
    & 4 & 1.36 & 3.46 & 115 & 0.16 \\
    & 5 & 0.80 & 2.90 & 108 & 0.13 \\
    & 6 & 0.21 & \textbf{2.31} & 101 & \textbf{0.07} \\
    & 7 & 0.74 & 2.84 & 94  & \textbf{0.07} \\
    & 8 & 0.83 & 2.93 & 84  & \textbf{0.07} \\
    & 9 & 2.30 & 4.40 & 131 & \textbf{0.07} \\
    \bottomrule
    \end{tabular}
    }
    \vspace{1mm}
    \caption{Computation times and $D_{\texttt{KS}}$ for different combinations of moments $K$ and mixture size $n$ on the Future Investor subject.}
    \label{table:computation_time}
\end{table}

\begin{table}[!htp]
    \centering
    \resizebox{0.32\textwidth}{!}{\begin{tabular}{@{}lllll@{}}
    \toprule
    \textbf{$a_i$} &\textbf{$T_{\texttt{opt}}$ (s)} & \textbf{$T_{\texttt{total}}$ (s)} & \textbf{$D_{\texttt{KS}}$} \\
    \midrule
    $i$ & 0.62 & 2.73 & 0.16 \\
    $2 \cdot i$ &1.95 & 4.05 & 0.15\\
    $3 \cdot i$ &0.74 & 2.84 & 0.13 \\
    $2^i$ & 1.45 & 3.55 & 0.12 \\
    $3^i$ & 0.21 & 2.31 & \textbf{0.07} \\
    $4^i$ & 1.37 & 3.47 & \textbf{0.07} \\
    \bottomrule
    \end{tabular}
    }
    \vspace{1mm}
    \caption{Comparison of different shape parameter sets $a_i$ (Future Investor subject, $k=5$, $n=6$)}
    \label{table:comparison}
\end{table}

\subsection{Hyper-parameters Effects on Accuracy and Runtime}\label{secHyperparameters}

\noindent\textbf{Number of Moments and Erlang Components}. 
Figure~\ref{fig:hyper} shows the plots of the approximate distribution constructed for the \texttt{Future Investor} subject computed with $K \in \{3, 4, 5\}$ and $n \in [3..9]$.
We selected \texttt{Future Investor} for this evaluation due to the presence of multiple, overlapping modalities in its PDF with a continuous reward space, which renders approximation particularly challenging.
The KS metrics in Table~\ref{table:computation_time} show that increasing the values of $K$ and $n$ can achieve higher accuracy -- the KS metric drops from 0.17 to 0.07. This is in line with the theoretical expectations from Section~\ref{secTheo}:
larger mixtures theoretically enhance approximation accuracy by allowing Erlang distributions to approximate any positive random variable, while more moments capture additional more features of the reward process, including multiple modalities and tail behavior. For example, $K=4$ and $n=6$, already achieve the best KS metric (0.07), which is also reported for $K=5$ and $n \in \{6, 7, 8, 9\}$.

However, numerical instability and limitations of the optimizer accuracy on this class of problems, means that, in practice, the quality of the (near-)optimal solutions found by the optimizer does not increase monotonically with $K$ and $n$. Although larger mixtures subsume smaller ones by setting some weights to zero, and generally achieve better performance due to increased expressiveness, the non-convexity of the problem means that the optimization is not guaranteed to converge to a single global optimum. In particular, larger $n$ in our experiments occasionally led to worse performance due to the accumulation of numerical errors across the larger number of free variables (e.g., see $K=4$ and $n \in \{7, 8, 9\}$, which performs marginally worse than with $n=6$). We plan to experiment with different optimization strategies in the future, including the use of more robust methods such as Expectation Maximization as in~\cite{Verbelen_Gong_Antonio_Badescu_Lin_2015} and studying different relaxation of the optimization problem to improve the stability of its numerical solution.

\vspace{1mm}
\noindent\textbf{Runtime.}  
Table~\ref{table:computation_time} summarizes the runtime statistics for each of the $(K,n)$ settings for \texttt{Future Investor} (2,899 states and 7,967 transitions). The total time ($T_{\texttt{total}}$, in seconds) shows the end-to-end time to construct the approximate distribution, which includes the computation of the moments of the reward distribution and the optimization process to infer the Erlang mixture. The optimization time ($T_{\texttt{opt}}$, in seconds) indicates the total time spent on optimizer's calls. The column \texttt{Iterations} shows the number of iterations of the optimization loop required to converge. While the number of iterations seems to weakly correlate with the value of $K$, the time for all the experiments remained consistently below 5 seconds on a Macbook Air M1 (8Gb of memory). Notably, both the evaluation of Cantelli's inequality for early termination (Eq.~\eqref{eqCantelliHigherMoments}) and of the CDF of the Erlang mixture (Eq.~\eqref{eqMixtureCDF}) only require negligible time to evaluate closed-form expressions.

\vspace{1mm}
\noindent\textbf{Heuristic Restriction of Shape Values.} In line with previous literature~\cite{Verbelen_Gong_Antonio_Badescu_Lin_2015,Lee_Lin_2012,verbelen2013phase}, we restricted the integer shape parameter values within a fixed set, in order to simplify the optimization problem (cf. Section~\ref{secDistributionApproximation}). Table~\ref{table:comparison} evaluates different heuristics on \texttt{Future Investor} with $K=5$ and $n=6$, which yielded an optimal KS metric for the subject: $a_i = i$ to the original assignment in~\cite{schweitzer1996stochastic}, which converges to the optimal approximation only for large mixtures; $a_i = 2 \cdot i$ and $a_i = 3 \cdot i$ use a linear spreading of shape values~\cite{verbelen2013phase}; finally, $a_i=2^i$, $a_i = 3^i$ and $a_i = 4^i$ use the exponential spreading of our heuristics. The intuition behind spreading the shape parameter values is to allow capturing distribution features that may be different from the head to the tail, especially in the presence of multiple modalities. The exponential spreading heuristic consistently achieved the best fitting. The heuristic $a_i=3^i$ consistently achieved the best performance across the board of our diverse subjects. Expectation Maximization has been used in literature to further refine the shape values around their initial heuristics points~\cite{verbelen2013phase}; we plan to experiment with such finer tuning in future work.

 \section{Related Work}\label{secRelated}
\noindent \textbf{Moment Matching Approximation.} Moment matching is a classical method for approximating density functions~\cite{john2007techniques, mnatsakanov2009recovery}. The main challenge of density approximation methods is to select an appropriate foundational distribution to represent the unknown distribution. Existing methods have focused on using powerful kernel density function~\cite{gavriliadis2012truncated} or Erlang mixtures~\cite{johnson1989matching, he2022continuous}, given their capability in recovering a wide variety of distributions. However, the accuracy of these approximations depends heavily on how well the selected function or kernel aligns with the underlying distribution. We utilize the Erlang mixtures, not only due to their flexibility but also because of its inherent connection to Markov models, which allows us to approximate the reward distribution with bounded error.

\vspace{1mm}
\noindent \textbf{Robust Model Checking.} Robustness is one of the most important concerns in model checking methods. Some simulation-based approaches provide probabilistic bounds with confidence guarantees of model correctness~\cite{herault2004approximate, pappagallo2020monte}. There are also several methods that provide evaluation beyond expected value with distributional information, including the quantile-based method~\cite{baier2014probabilistic} and histogram approximation method~\cite{elsayed2024distributional}, but could lead to unbounded information loss due to the coarse approximation. In contrast, our approach provides smoother approximation while avoiding discretization errors and offering guaranteed accuracy in both discrete and continuous domains.

\vspace{1mm}
\noindent \textbf{Distributional Reinforcement Learning.} Recent work in distributional reinforcement learning has focused on estimating the full return distribution instead of just its expectation. a categorical histogram approximation method has been prosed in~\cite{bellemare2017distributional} to compute return distributions using discrete bins, optimizing the Wasserstein distance between estimated and target distributions. Extensions of this work, such as quantile regression (QR-DQN)~\cite{dabney2018distributional} and implicit quantile networks (IQN)~\cite{dabney2018implicit}, improve on this by learning the quantile-based approximation, offering greater flexibility and risk-sensitive policies. Our method can be further extended to the RL domain, providing precise estimation and flexibility in continuous reward environments than histogram methods, with a robust verification that can be used as guidance for distributional policy iteration.

 \section{Conclusion}\label{secConclusion}
In this work, we addressed the limitations of current PMC methods that rely on expected values, by proposing a moment-matching approach using mixtures of Erlang distributions. This approach enables a more precise approximation of cumulative reward distributions in DTMCs by capturing multi-modality and higher-order statistical features. We enhance system verification by incorporating the full distribution of outcomes, enabling better decision support for adaptive systems.
Next, we aim to extend our approach to robust model checking and learning in MDPs, where distributional robustness could guide policy search. We plan to leverage the prior knowledge of the accumulated reward distribution in MDPs, as demonstrated in this work, to provide a robust alternative to histogram-based approximation, which may struggle in distributional RL scenarios involving sparse or dense reward regions and continuous reward space. We also aim to explore the applicability of this approach in distributional learning settings beyond distributional model checking, to enhance policy iteration process and evaluate its potential for synthesizing more robust policies under uncertainty in practical applications.

\cleardoublepage
\bibliographystyle{acm}
\bibliography{bib}

\end{document}